\documentclass[letterpaper]{article} % DO NOT CHANGE THIS
\usepackage{aaai2026}  % DO NOT CHANGE THIS
\usepackage{times}  % DO NOT CHANGE THIS
\usepackage{helvet}  % DO NOT CHANGE THIS
\usepackage{courier}  % DO NOT CHANGE THIS
\usepackage[hyphens]{url}  % DO NOT CHANGE THIS
\usepackage{graphicx} % DO NOT CHANGE THIS
\usepackage{natbib}  % DO NOT CHANGE THIS AND DO NOT ADD ANY OPTIONS TO IT
\usepackage{caption} % DO NOT CHANGE THIS AND DO NOT ADD ANY OPTIONS TO IT
\usepackage{algorithm}
\usepackage{algorithmic}

\usepackage{amssymb}
\usepackage{amsmath}
\usepackage{amsthm}
\usepackage{subcaption}
\usepackage{makecell}
\usepackage{multirow}
\usepackage{xcolor}
\usepackage{subcaption}
\usepackage[capitalize,noabbrev]{cleveref}

\theoremstyle{plain}
\newtheorem{theorem}{Theorem}

\newtheorem{lemma}{Lemma}

\theoremstyle{definition}

\theoremstyle{remark}

\usepackage{newfloat}
\usepackage{listings}
\DeclareCaptionStyle{ruled}{labelfont=normalfont,labelsep=colon,strut=off} % DO NOT CHANGE THIS
\floatstyle{ruled}
\newfloat{listing}{tb}{lst}{}
\floatname{listing}{Listing}
\title{Test-driven Reinforcement Learning in Continuous Control}
\author{
    Zhao Yu\equalcontrib\textsuperscript{\rm 1},
    Xiuping Wu\equalcontrib\textsuperscript{\rm 2},
    Liangjun Ke\textsuperscript{\rm 1}\thanks{Corresponding author},
}
\affiliations{
    \textsuperscript{\rm 1} The State Key
Laboratory for Manufacturing Systems Engineering, School of Automation
Science and Engineering, \\
Xi’an Jiaotong University, No.28 Xianning West Road, Xi'an, Shaanxi 710049, P.R. China
\\
    \textsuperscript{\rm 2}Southampton Business School, University of Southampton, Highfield Campus, SO17 1BJ, United Kingdom\\

    yuzhaokz@stu.xjtu.edu.cn,
    xiuping.wu@soton.ac.uk,
    keljxjtu@xjtu.edu.cn
}

\usepackage{bibentry}
\begin{document}

\maketitle

\begin{abstract}
Reinforcement learning (RL) has been recognized as a powerful tool for robot control tasks. RL typically employs reward functions to define task objectives and guide agent learning.
However, since the reward function serves the dual purpose of defining the optimal goal and guiding learning, it is challenging to design the reward function manually, which often results in a suboptimal task representation. 
To tackle the reward design challenge in RL, inspired by the satisficing theory, we propose a Test-driven Reinforcement Learning (TdRL) framework. In the TdRL framework, multiple test functions are used to represent the task objective rather than a single reward function. Test functions can be categorized as pass-fail tests and indicative tests, each dedicated to defining the optimal objective and guiding the learning process, respectively, thereby making defining tasks easier.
Building upon such a task definition, we first prove that if a trajectory return function assigns higher returns to trajectories closer to the optimal trajectory set, maximum entropy policy optimization based on this return function will yield a policy that is closer to the optimal policy set. Then, we introduce a lexicographic heuristic approach to compare the relative distance relationship between trajectories and the optimal trajectory set for learning the trajectory return function. Furthermore, we develop an algorithm implementation of TdRL.
Experimental results on the DeepMind Control Suite benchmark demonstrate that TdRL matches or outperforms handcrafted reward methods in policy training, with greater design simplicity and inherent support for multi-objective optimization.
We argue that TdRL offers a novel perspective for representing task objectives, which could be helpful in addressing the reward design challenges in RL applications.
\end{abstract}

\begin{links}
    \link{Code}{https://github.com/KezhiAdore/TdRL}
    % \link{Extended version}{https://aaai.org/example/}
\end{links}

\section{Introduction}
\label{sec:introduction}
For an intelligent agent, it is crucial to specify the objective that needs to be achieved \cite{silver_2021, silver_2025}.
In classical reinforcement learning (RL), the objective is typically specified through the reward functions, which steer the agent toward desired behaviors. \cite{sutton_rl_2018}
Crucially, reward design accounts for both defining optimal behavior and guiding the learning process, making it inherently challenging \cite{rajagopal_2023,knox_misdesign_2023,booth_2023}. 
Crafting effective reward functions typically demands domain-specific expertise and reward design experience \cite{knox_misdesign_2023,booth_2023}. 
Nevertheless, the handcrafted reward often represents an imperfect proxy for the true optimization objective, as experts typically view reward as a direct evaluation of the relative goodness of each state-action pair instead of evaluating trajectories.\cite{booth_2023}. 
This mismatch induces fundamental challenges in reinforcement learning, such as reward hacking \cite{amodei_2016}.

To circumvent the challenges associated with manual reward engineering, several approaches have been developed, mainly including Preference-based RL (PbRL, \citeauthor{pbrl_2017} \citeyear{pbrl_2017}), Inverse RL (IRL, \citeauthor{maxent_irl_2008} \citeyear{maxent_irl_2008}), and reward functions generated by Large Language Models (LLMs, \citeauthor{rewardlm_2023} \citeyear{rewardlm_2023}).

PbRL learns reward functions or directly optimizes policies based on human preference feedback, which avoids manual reward design and has demonstrated promising performance in robot control and LLM alignment \cite{pbrl_2017, pbrl_llm_2022}. However, PbRL's human-centric nature introduces several limitations \cite{pbrl_open_problems_2023}.
First, inherent biases in human preference data necessitate dedicated techniques to address them. 
Second, the reliance on human judgments may constrain the agent's behavior within human cognitive boundaries. 
% Third, the requirement for large-scale human annotations and the subjective nature of preference pose significant reproducibility challenges. 
IRL infers reward functions from expert demonstrations, thereby avoiding manual reward designing.
However, IRL typically demands extensive expert demonstration data, and the inferred reward functions typically exhibit poor generalization when evaluated on out-of-distribution data \cite{irl_survey_2021}.
Recent studies have demonstrated that LLMs can outperform humans in designing reward functions \cite{rewardlm_2023,l2r_2023,ellm_2023}. However, these approaches depend on a human-specified domain knowledge for reward design,  and typically require extensive training feedback to polish the reward function \cite{llm_survey_2025}.

Furthermore, real-world applications typically require agents to simultaneously optimize multiple objectives \cite{peter_2023}. This multi-objective optimization presents significant challenges for reward function design, particularly in balancing different objectives.

When solving real-world multi-objective tasks, humans often do not pursue the optimal solution in a certain metric but rather seek a \textit{satisficing solution} \cite{satisficing_1947} across multiple objectives. For example, when driving, people do not blindly minimize time consumption, but rather reach the destination within a certain time while ensuring safety, comfort, and compliance with regulations.

Inspired by this, we propose a test-driven reinforcement learning (TdRL) framework. In the TdRL framework, the agent's objective is passing given tests instead of maximizing the cumulative return, which implies that TdRL seeks to obtain a satisficing solution across multiple objectives rather than an optimal solution in a single metric.
% defined by one or more test functions instead of a single predefined reward function. 
Test functions take trajectories as input and output the test results. According to their functionality, test functions can be categorized into two types: pass-fail tests and indicative tests. 
Pass-fail tests evaluate whether the policy meets the required criteria, while indicative tests provide informative guiding signals for policy learning. Pass-fail tests and indicative tests correspond to defining optimal behavior and guiding the learning process, respectively.

Specifically, the goal of TdRL is to train a policy such that its interaction trajectories with the environment pass all given pass-fail tests. 
To solve this problem, we first prove that if a trajectory return function assigns higher returns to trajectories closer to the optimal trajectory set (where the trajectory passes all given pass-fail tests), maximum entropy policy optimization based on this return function will yield a policy that is closer to the optimal policy set (where the trajectories generated by the policy pass all given pass-fail tests).
Then, we introduce a lexicographic heuristic approach to compare the relative distance relationship between trajectories and the optimal trajectory set. 
Additionally, we develop an algorithm implementation of TdRL.
%  with return decomposition, which ensures compatibility between TdRL and existing RL methods.

% The goal of TdRL is to train a policy such that its interaction trajectories with the environment satisfy all given pass-fail test functions, rather than optimizing with respect to a handcrafted scalar reward function. Pass-fail test functions provide meaningful feedback only when the policy's interaction trajectories meet predefined criteria. This resembles sparse rewards in reinforcement learning, with the distinctions that (1) multiple pass-fail test functions may exist, whereas sparse rewards are typically defined by a single reward function, and (2) pass-fail test functions operate on entire trajectories, whereas sparse rewards are computed on state-action pair.

% Pass-fail tests specify the optimization objective, but their sparsity may impede learning efficiency. Indicative tests offer more informative feedback as they evaluate trajectories along several predefined metrics, which could be used to accelerate reinforcement learning convergence. 
% 点出互补的作用，不要比较
% 大模型奖励设计
% 基于奖励学习的是第一个，关系和区别，通用性更好，理论框架
% 利用 test 指导RL在应用中是取得成效的，有一些萌芽，但是缺乏一套通用的方案，系统性的工作
% reward learning 
% 提一下方法内容，怎么利用test进行学习补充到intro里面

% Furthermore, indicative tests and pass-fail tests exhibit correlation—optimizing one or more indicative tests enables the agent to acquire policies that meet corresponding pass-fail tests. For instance, when training a robot dog to run, average speed can function as an indicative test for trajectories, while exceeding a speed threshold serves as a pass-fail test.

The TdRL framework provides several key benefits. First, test functions operate on trajectories instead of state-action pairs, mitigating designer-induced bias \cite{booth_2023}. 
Second, the task objective is represented by pass-fail tests rather than a scalar reward function, naturally accommodating multi-objective optimization. 
% Third, test functions can evaluate the policy across multiple criteria and difficulty levels, facilitating a finer task specification. 
Finally, test results can evaluate policy performance more accurately than cumulative returns. The main contributions of this paper are summarized as follows:

\begin{itemize}
    \item To address the reward design challenge, we propose a test-driven reinforcement learning framework, where the task objective is represented by several test functions instead of a single reward function.
    \item We propose sufficient conditions for trajectory return functions to guarantee policy convergence to the optimal policy set, and present a lexicographic heuristic approach for constructing return functions based on trajectory testing results.
    \item We develop an algorithm implementation of TdRL and conduct experiments on the DeepMind Control Suite benchmark to show the benefits of TdRL.
\end{itemize}

\section{Related Work}
\label{sec:related_work}

\textbf{Test-driven reward design} has been adapted to RL to ensure robustness and correctness recently \cite{tddreward_2022, tddirl_2024}. Inspired by the test-driven development principle \cite{tdd_2002}, \citeauthor{tddreward_2022} \shortcite{tddreward_2022} proposed an approach that uses test cases to guide reward design, iteratively refining the policy until all tests are passed. \citeauthor{tddirl_2024} \shortcite{tddirl_2024} extended this idea to inverse reinforcement learning, replacing expert demonstrations with scenario-based testing to learn cost functions via Bayesian inference. \citeauthor{testimportance_2024} \shortcite{testimportance_2024} introduced a state-importance-driven testing method for deep RL, prioritizing high-impact states to efficiently detect policy vulnerabilities. These approaches demonstrate the potential of test-driven frameworks to enhance RL in applications.

\textbf{Preference-based reinforcement learning (PbRL)} has emerged to address the challenges of reward design. Instead of relying on handcrafted reward functions, PbRL \cite{pbrl_2017} learns reward models from human preferences over agent behaviors, which in turn guides policy optimization, achieving superior performance compared to handcrafted rewards. 
PEBBLE method \cite{pebble_2021} enhances sample efficiency by integrating off-policy learning and unsupervised pre-training. 
% \citeauthor{bpref_2021} \shortcite{bpref_2021} also proposed B-Pref, a benchmark for PbRL that evaluates robustness under various human feedback conditions, including noisy and sparse preferences. 
To further reduce human feedback requirements, a semi-supervised method named SURF \cite{surf_2022} combines data augmentation with PbRL to leverage both labeled and unlabeled experience. More recently, \citeauthor{heron_2024} \shortcite{heron_2024} proposed a hierarchical reward modeling framework called HERON that structures preference learning based on the importance of feedback signals, improving performance in sparse-reward scenarios.

% \textbf{LLM-Guided reward design} has opened new possibilities for automating reward design in RL. \citeauthor{rewardlm_2023} \shortcite{rewardlm_2023} demonstrated that LLMs can assess trajectory quality in text-based tasks, converting natural language evaluations into reward signals. Extending to robotics, \citeauthor{l2r_2023} \shortcite{l2r_2023} proposed a method where LLMs generate reward function parameters from natural language instructions, which are then implemented using model predictive control (MPC) for robotic system control. Similarly, \citeauthor{ellm_2023} \shortcite{ellm_2023} utilized LLMs to guide RL pretraining by generating exploration goals that align with commonsense reasoning. \citeauthor{lafite_2023} \shortcite{lafite_2023} introduced Lafite-RL, a framework that leverages real-time LLM feedback to accelerate learning in robotic manipulation tasks. A notable advancement is Eureka \cite{eureka_2024}, which employs LLMs to design reward function codes based on environment code, and refine reward functions using the evolutionary algorithm, achieving human-level reward design without manual intervention.

Test-driven methods have been applied in RL recently, but they often require \textit{manual} design of approaches to process test results. We introduce the first theoretical framework for test-driven reinforcement learning that is learning-based.
% formally defining its optimization problem and providing theoretical justification for optimizing it via a progressive reward function. 
Algorithmically, we adapt the PbRL paradigm of learning the return function from trajectory comparison.

\section{Preliminaries}
\label{sec:preliminary}

A Markov Decision Process (MDP) could be described by a tuple $\mathcal M = <\mathcal S, \mathcal A, \mathcal P, r, \gamma>$ where $\mathcal S$ is the state space, $\mathcal A$ is the action space, $\mathcal P$ is the transition probability function, $r\in\mathcal R:\mathcal S\times\mathcal A\rightarrow \mathbb R$ is the reward function, and $\gamma$ is the discount factor. The goal of reinforcement learning \cite{sutton_rl_2018} is to find a policy $\pi\in\Pi:\mathcal S\times \mathcal A\rightarrow [0,1]$ that maximizes the expected cumulative reward:
\begin{equation}
    \pi^* = \arg\max_{\pi} \mathbb{E}_{\pi} \left[ \sum_{t=0}^{\infty} r(s_t,a_t) \right].
    \label{eq:pi_star}
\end{equation}

A trajectory $\tau\in\mathcal T$ represents a sequence of state-action pairs generated through policy-environment interactions, beginning at some initial state. The trajectory return function $R$ evaluates a trajectory by summing rewards across all constituent state-action pairs:$R(\tau)=\sum_{(s,a)\in\tau} r(s,a)$.
Given these definitions, the reinforcement learning optimization objective is formally expressed as:
\begin{equation}
\pi^* = \arg\max_{\pi} \mathbb{E}_{\tau\sim\pi} \left[ R(\tau) \right].
\end{equation}

\subsection{Maximum Entropy Reinforcement Learning}

Maximum Entropy Reinforcement Learning (MERL, \citeauthor{sac_2018} \citeyear{sac_2018}) extends the standard reinforcement learning framework by incorporating an entropy maximization objective. This approach simultaneously optimizes for both cumulative reward maximization and policy entropy, thereby improving exploration capacity and algorithmic robustness.The objective of MERL could be written as:
\begin{equation}
    \max_\pi \mathbb{E}_{\tau \sim \pi} \left[ \sum_{t=0}^T r(s_t, a_t) + \alpha H(\pi(\cdot|s_t)) \right],
    \label{eq:max_entropy}
\end{equation}
where $H(\pi(\cdot|s_t))$ denotes the policy's entropy at state $s_t$, and $\alpha$ controls the trade-off between the reward and entropy terms.
When policy $\pi_1$ is optimized via maximum entropy reinforcement learning with respect to the trajectory return function $R(\tau)=\sum_{(s,a)\in\tau} r(s,a)$, the resulting policy $\pi_2$ follows:
\begin{equation}
    \pi_2(\tau) = \frac{1}{Z} \pi_1(\tau) \exp\left( \frac{1}{\alpha} R(\tau) \right), \label{eq:maxent}
\end{equation}where $Z$ is the partition function, $\pi(\tau)$ represents the probability of generating trajectory $\tau$ using policy $\pi$. The proof of \cref{eq:maxent} is in Appendix \ref{app:proof_eq:maxent}.

\subsection{Preference-based Reinforcement Learning}
Preference-based reinforcement learning \cite{pbrl_2017} is a class of reinforcement learning algorithms that learns a policy by comparing the rewards of different trajectories. 
% A trajectory $\tau=\{s_0,a_0,s_1,a_1,...\}$ is a time sequence composed of state-action pairs, which could be finite or infinite. We use $\mathcal{T}$ to denote the set of all trajectories in a given environment, and use $R(\tau)=\sum_{(s_t,a_t)\in\tau} r(s_t,a_t)$ to denote the total reward of a trajectory $\tau$.
A preference relation $\succ$ is defined on $\mathcal{T}$ as follows:
\begin{equation}
    \tau_1 \succ \tau_2 \Leftrightarrow R(\tau_1)\geq R(\tau_2).
    \label{eq:succ}
\end{equation}
Following the Bradley-Terry model \cite{bradley_1952}, we could use a reward function estimate $\hat r$ as the preference predictor, as $\hat r$ could be viewed as a latent factor to explain the preference. We assume the preference probability of a trajectory depends exponentially on its return:
\begin{equation}
    \hat P[\tau_1\succ\tau_2] = \frac{\exp(\hat R(\tau_1))}{\exp(\hat R(\tau_1))+\exp(\hat R(\tau_2))},
    \label{eq:P_succ}
\end{equation}

For a given dataset $\mathcal D$ of triples $(\tau_1,\tau_2,\mu)$, where $\tau_1$ and $\tau_2$ are different trajectories and $\mu$ is a distribution over $\{1,2\}$ indicating which trajectory is preferred, we could learn the reward function by minimize the following loss function :
\begin{align}
    \mathcal L(\hat r) = &- \sum_{(\tau_1,\tau_2,\mu)\in\mathcal D} \mu(1) \log \hat P[\tau_1\succ\tau_2] \notag \\
    &+ \mu(2) \log (1-\hat P[\tau_1\succ\tau_2]).
    \label{eq:L_pbrl}
\end{align}

\section{Method}

In this section, we introduce Test-driven reinforcement learning. We first illustrate the notation and goal in TdRL, then we establish sufficient conditions for trajectory return functions to guarantee policy convergence to the optimal policy set. Additionally, we introduce a lexicographic heuristic approach to compare trajectories for learning trajectory return function. Finally, we detail the implementation of the TdRL algorithm. \cref{fig:algo_prodecure} demonstrates the main procedure of the TdRL algorithm. It follows an iterative procedure mainly consisting of four stages:
\begin{itemize}
    \item \textit{Collection Trajectory}: The agent interacts with the environment to collect trajectories.
    \item \textit{Return Learning}: The return function is updated by the comparison results of trajectories.
    \item \textit{Reward Learning}: The reward function is updated and rewards are recalculated for transitions in the replay buffer.
    \item \textit{Policy Optimization}: The policy network is optimized using transitions in the replay buffer.
\end{itemize}

\label{sec:method}
\begin{figure}[tbp]
\centering
\includegraphics[width=\columnwidth]{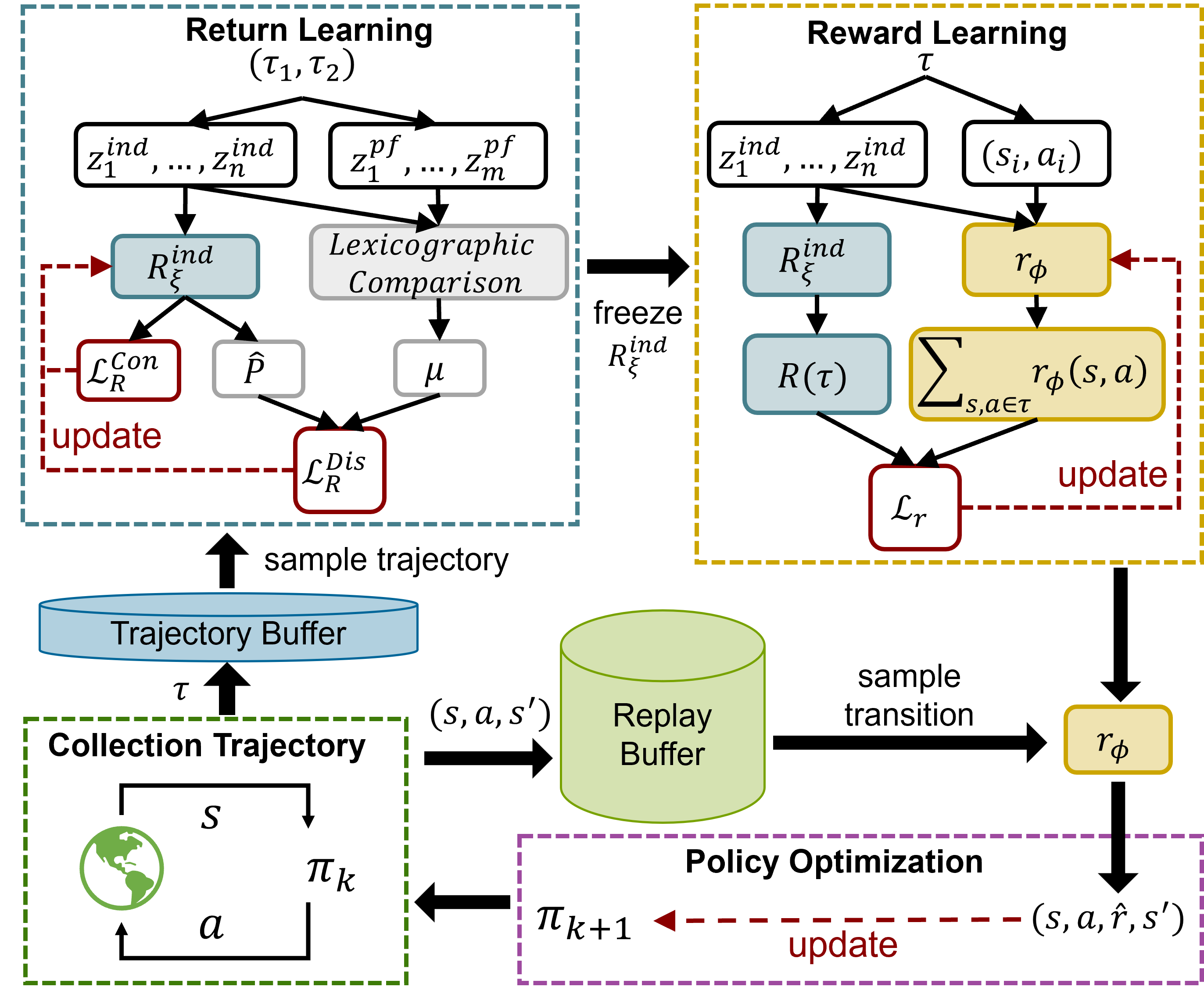}
\caption{The main procedure of the TdRL algorithm.}
\label{fig:algo_prodecure}
\end{figure}

\subsection{Notation and Goal}

TdRL aims to find a policy whose interaction trajectories with the environment pass all given pass-fail tests. Test functions map a trajectory to a test result. Specifically, test functions can be categorized into two types based on their functionality: pass-fail tests and indicative tests. A pass-fail test outputs a binary value judging whether a trajectory meets the required criteria. In contrast, an indicative test outputs a real number, quantifying the performance of a trajectory in a specific metric. Formally, we use $z^{pf}\in \mathcal Z^{pf}: \mathcal T\rightarrow \{0,1\}$ to denote pass-fail test function, and $z^{ind}\in \mathcal Z^{ind}:\mathcal T\rightarrow \mathbb R$ to denote indicative test function.

The task objective in TdRL can be formulated as $<\mathcal S,\mathcal A,\mathcal P,Z^{pf},Z^{ind},\gamma>$, where $Z^{pf}=\{z^{pf}_1,...,z^{pf}_m\}$ represents a set of $m$ pass-fail test functions and $Z^{ind}=\{z^{ind}_1,...,z^{ind}_n\}$ represents a set of $n$ indicative test functions. The goal of TdRL is to find a policy $\pi^*$ such that the resulting trajectory $\tau$ from the interaction between $\pi^*$ and the environment satisfies:
\begin{equation}
    \mathbb E_{\tau\sim\pi^*} \left[\sum_{i=1}^m z^{pf}_i(\tau)\right]=m.
\end{equation}
% However, in many environments and tasks, a policy that strictly satisfies the above conditions may not exist, but one can find a policy under which most generated trajectories pass most pass-fail tests, i.e., $E_{\tau \sim \pi^*}[\sum_{i=1}^{m} z_i^{pf}(\tau)] > m \cdot (1 - \epsilon)$, where $\epsilon > 0$ is a sufficiently small constant.

Following the above definition, we define the optimal trajectory set as the collection of all trajectories that pass all pass-fail tests, denoted as:
\begin{equation}
    \tilde{\mathcal{T}} = \{\tau | \tau\in\mathcal T, \sum_{i=1}^m z^{pf}_i(\tau) = m\}.
\end{equation}
Let $\tilde{\mathcal{T}}_i = \{\tau | \tau\in\mathcal T, z^{pf}_i(\tau) = 1\}$ denote the set of trajectories that pass the pass-fail test $z^{pf}_i$. Then, we have $\tilde{\mathcal{T}} = \cap_{i=1}^m \tilde{\mathcal{T}}_i$.

Similarly, we define the optimal policy set as the collection of all policies whose interaction with the environment generates trajectories that pass all pass-fail tests, denoted as:
\begin{equation}
    \tilde{\Pi} = \{\pi | \pi \in \Pi, \mathbb{E}_{\tau \sim \pi} [\sum_{i=1}^m z^{pf}_i(\tau)] = m\}.
\end{equation}
Let $\tilde{\Pi}_i = \{\pi | \pi \in \Pi, \mathbb{E}_{\tau \sim \pi} [z^{pf}_i(\tau)] = 1\}$ denote the set of policies whose interaction with the environment generates trajectories that pass the pass-fail test $z^{pf}_i$. Then, we have $\tilde{\Pi} = \cap_{i=1}^m \tilde{\Pi}_i$.

\subsection{Test Functions to Return Function}

The goal of TdRL is to find a policy $\pi^*$ that generates trajectories passing all pass-fail tests. In other words, we need to find a policy in $\tilde{\Pi}$. To achieve this, we could construct a trajectory return function $R$ such that policy optimization based on this function could converge to a policy belonging to $\tilde{\Pi}$.
Let $d(\tau_1,\tau_2)$ denote the distance between trajectories $\tau_1$ and $\tau_2$ in the trajectory space, and the distance between a trajectory $\tau$ and a trajectory set $\hat{\mathcal{T}}$ is defined as:
\begin{equation}
 d(\tau,\hat{\mathcal{T}}) = \min_{\tau'\in\hat{\mathcal{T}}} d(\tau,\tau'). 
 \label{eq:d_tau_hatT}
\end{equation}

For a policy $\pi$, let $P_\pi:\mathcal{T}\rightarrow[0,1]$ denote the distribution of trajectories generated through interaction with the environment, satisfying $\int_\tau P_\pi(\tau)d\tau=1$. The distance between policies $\pi_1$ and $\pi_2$ is defined by the Wasserstein-$p$ distance \cite{villani_2008} between their corresponding trajectory distributions $P_{\pi_1}$ and $P_{\pi_2}$:
\begin{align}
    &d(\pi_1,\pi_2) = W_p(P_{\pi_1}, P_{\pi_2}) \notag \\
    &= \left( \inf_{\gamma \in \Gamma(P_{\pi_1}, P_{\pi_2})} \int_{\mathcal{T} \times \mathcal{T}} d(\tau_1, \tau_2)^p \, d\gamma(\tau_1, \tau_2) \right)^{\frac{1}{p}},
\end{align}
where $\Gamma(P_{\pi_1}, P_{\pi_2})$ denotes the set of all joint probability distributions whose marginal distributions are $P_{\pi_1}$ and $P_{\pi_2}$, respectively.
The distance between policy $\pi$ and policy set $\hat{\Pi}$ is defined as:
\begin{equation}
    d(\pi,\hat{\Pi}) = \min_{\pi'\in\hat{\Pi}} d(\pi,\pi').
\end{equation} 

\begin{theorem}
    \label{theorem:R_pi}
    If there exists a trajectory return function $R(\tau)$ that is monotonically non-increasing with respect to the distance between a trajectory $\tau$ and the optimal trajectory set $\tilde{\mathcal{T}}$, such that:
$$
d(\tau_1, \tilde{\mathcal{T}}) \leq d(\tau_2, \tilde{\mathcal{T}}) \implies R(\tau_1) \geq R(\tau_2).
$$
Suppose policy $\pi_2$ is obtained by optimizing policy $\pi_1$ using a maximum entropy algorithm with respect to $R$, 
Then, policy $\pi_2$ is closer to the optimal policy set $\tilde{\Pi}$ than $\pi_1$:
$$
d(\pi_1, \tilde{\Pi}) \geq d(\pi_2, \tilde{\Pi}).
$$
\end{theorem}
The proof of \cref{theorem:R_pi} is shown in \cref{app:proof_R_pi}.
\cref{theorem:R_pi} demonstrates that if there exists a trajectory return function $R$ that assigns higher values to trajectories closer to the optimal trajectory set $\tilde{\mathcal T}$, then performing maximum entropy policy optimization on this return function will yield a policy that is closer to the policy set $\tilde\Pi$ than the original policy.
To obtain a trajectory reward function $R$ that satisfies the above conditions, we consider two key issues: how to construct the trajectory return function and how to design a loss function for learning such a return function.

There typically exist numerous natural indicative signals in environments. The key challenge lies in effectively combining these indicative signals into a reward signal that can properly guide the learning process \cite{Juechems_2019,silver_2025}. Motivated by this insight, our approach does not directly construct a function mapping trajectories to returns. Instead, we regard indicative test results as indicative signals, and we utilize indicative test functions to construct a trajectory return function. Appendix \ref{app:indicative} provides a detailed interruption about the indicative test function combination.

Specifically, we employ a fully connected network that takes an input vector of dimension $n$ and produces a scalar output parameterized by $\xi$ to construct a return mapping function $R^{ind}_\xi$. Then, the trajectory return function $R$ could be constructed as a composite function of $R^{ind}_\xi$ and the $n$ indicative test functions:
\begin{equation}
    R(\tau) = R^{ind}_\xi(z^{ind}_1(\tau), z^{ind}_2(\tau), ..., z^{ind}_n(\tau)).
\end{equation}

Following the Bradley-Terry model \cite{bradley_1952}, if we consider the distance to the optimal trajectory set as a measure of trajectory quality, we can estimate the probability of distance relationships between trajectories by the return function. For convenience, we define $\tilde d(\tau)$ as the distance between trajectory $\tau$ and the optimal trajectory set $\tilde{\mathcal{T}}$. Then, the probability of $\tau_1$ being closer to $\tilde{\mathcal{T}}$ than $\tau_2$ can be estimated as:
\begin{equation}
    \hat{P}[\tilde{d}(\tau_1) < \tilde{d}(\tau_2)] = \frac{\exp(R(\tau_1))}{\exp(R(\tau_1)) + \exp(R(\tau_2))}.
\end{equation}

Suppose $\mu \in \{0,0.5,1\}$ represents the true probability that $\tau_1$ is closer to $\tilde{\mathcal{T}}$ than $\tau_2$, where $\mu=0$ means $\tilde{d}(\tau_1) > \tilde{d}(\tau_2)$, $\mu=1$ means $\tilde{d}(\tau_1) < \tilde{d}(\tau_2)$, and $\mu=0.5$ means $\tilde{d}(\tau_1) = \tilde{d}(\tau_2)$. Then, we could construct the following distance-based loss function:
\begin{align}
    \mathcal{L}_R^{Dis} = 
    &-\sum_{(\tau_1,\tau_2,\mu)\in\mathcal{D}} \left[ \mu \cdot \log \hat{P}[\tilde{d}(\tau_1) < \tilde{d}(\tau_2)]\right. \notag \\
    &+ \left.(1-\mu) \cdot \log \left(1 - \hat{P}[\tilde{d}(\tau_1) < \tilde{d}(\tau_2)]\right) \right], \label{eq:loss_return_dis}
\end{align}
where $\mathcal{D}$ is a dataset of trajectory pairs $(\tau_1,\tau_2)$ and their corresponding probabilities $\mu$.
Furthermore, to ensure numerical stability during return function learning, we introduce a penalty term:
\begin{equation}
    \mathcal{L}_R^{Penalty}=\sum_{(\tau_1,\tau_2,\mu)\in\mathcal{D}}  \sum_{i\in\{1,2\}}\left(R(\tau_i) - \tilde{R}(\tau_i)\right)^2, 
    \label{eq:loss_return_con}
\end{equation}
where $\tilde{R}$ denotes the trajectory return values computed before network updating.

Then, we decompose the learned trajectory return function into a state-action reward function for policy learning.
We parameterize the reward function as $r_\phi(s,a)$ and construct the reward learning loss:
\begin{equation}
    \mathcal L_r=\sum_{\tau\in \mathcal D} \left[R(\tau)-\sum_{(s,a)\in\tau}r_\phi(s,a) \right]^2 . 
    \label{eq:loss_reward}
\end{equation}

% which yields:
% $$
% R(\tau)=\sum_{(s,a)\in\tau}r_\phi(s,a)
% $$
% Based on this formulation, we construct the reward learning loss:

% The complete loss function of $R^{ind}_\xi$ is defined as:
% $$
% \mathcal{L}_R = \mathcal{L}_R^{Dis} + \lambda \cdot \mathcal{L}_R^{Penalty},
% $$
% where $\lambda > 0$ is a small constant that controls the update magnitude.

By minimizing the loss $\mathcal{L}_R^{Dis}$ and $\mathcal{L}_R^{Penalty}$, we can learn the return mapping function $R^{ind}_\xi$. Then we could decompose the trajectory return to state-action reward by learning a reward function through minimizing the loss $\mathcal L_r$.

Computing $\mathcal{L}_R^{Dis}$ requires the relative distance relationships between trajectories.
However, during the learning process, the trajectories contained in the optimal trajectory set $\tilde{\mathcal{T}}$ remain unknown, making it intractable to directly compute the distances $\tilde{d}(\tau)$.
Notably, we do not require exact values of $\tilde{d}(\tau)$ when computing the loss $\mathcal{L}_R^{Dis}$, but only need to determine the relative distance relationships between trajectories, specifically, whether $\tilde{d}(\tau_1) > \tilde{d}(\tau_2)$, $\tilde{d}(\tau_1) < \tilde{d}(\tau_2)$ or $\tilde{d}(\tau_1) = \tilde{d}(\tau_2)$. In the next section, we will introduce a lexicographic heuristic approach to construct the relative distance relationships between trajectories.

\subsection{Lexicographic Trajectory Comparison}
\label{sec:lexicographic}

For trajectories $\tau_1$ and $\tau_2$, we need to compare their distances to the optimal trajectory set $\tilde{\mathcal{T}}$, i.e., $\tilde{d}(\tau_1)$ and $\tilde{d}(\tau_2)$. However, directly computing $\tilde{d}(\tau)$ is infeasible as $\tilde{\mathcal T}$ is unknown. Under such limited information conditions, we employ a lexicographic heuristic \cite{lexicographic_2020} approach to compare the distances. The lexicographic heuristic is efficient, practical, and adaptable, which is a fast and frugal strategy originally from human decision-making \cite{gigerenzer_1996,kk_2013,kk_2021}. Specifically, the lexicographic method sequentially evaluates information in priority order and makes decisions based on the first criterion that meets predefined thresholds (e.g., exceeding certain values).

Building upon the definitions of pass-fail tests and indicative tests, along with the relationship between $\tilde{\mathcal{T}_i}$ and $\tilde{\mathcal{T}}$, we could derive the following priors (The interpretation of the comparison priors is detailed in Appendix \ref{app:interpertation}.):
\begin{itemize}
    \item Pass-fail tests take precedence over indicative tests
    \item A trajectory satisfying more pass-fail tests is closer to $\tilde{\mathcal{T}}$
    \item A trajectory passing more challenging tests (corresponding to smaller $\tilde{\mathcal{T}_i}$) is closer to $\tilde{\mathcal{T}}$
    \item All trajectories within $\tilde{\mathcal{T}}$ have zero distance to $\tilde{\mathcal{T}}$
    \item Under-optimized indicators should be prioritized
\end{itemize}

Building upon the established priors, we could compute the probability $\mu$ that $\tau_1$ is closer to $\tilde{\mathcal{T}}$ than $\tau_2$ through the following lexicographical  procedure (more details are shown in Appendix \ref{app:interpertation}):
\begin{enumerate}
    \item If $\sum_{i=1}^m z^{pf}_i(\tau_j) = m, \forall \tau\in\{\tau_1,\tau_2\}$, return $\mu=0.5$.
    \item Compare pass-fail test passing count:
    \begin{itemize}
        \item If $\sum_{i=1}^m z^{pf}_i(\tau_1) > \sum_{i=1}^m z^{pf}_i(\tau_2)$, return $\mu=1$
        \item If $\sum_{i=1}^m z^{pf}_i(\tau_1) < \sum_{i=1}^m z^{pf}_i(\tau_2)$, return $\mu=0$
    \end{itemize}
    \item Sort pass-fail tests in ascending order of historical pass rates (i.e., descending difficulty) as $\{z^{pf}_{k_1},...,z^{pf}_{k_m}\}$. Sequentially compare $z^{pf}_{k_i}(\tau_1)$ and $z^{pf}_{k_i}(\tau_2)$.
    \begin{itemize}
        \item If $z^{pf}_{k_i}(\tau_1) > z^{pf}_{k_i}(\tau_2)$, return $\mu=1$
        \item If $z^{pf}_{k_i}(\tau_1) < z^{pf}_{k_i}(\tau_2)$, return $\mu=0$
    \end{itemize}
    \item Sort indicative tests in descending order of the skewness of historical test results (least-optimized first) as $\{z^{ind}_{l_1},...,z^{ind}_{l_n}\}$. Sequentially compare $z^{ind}_{l_i}(\tau_1)$ and $z^{ind}_{l_i}(\tau_2)$.
    \begin{itemize}
        \item If $z^{ind}_{l_i}(\tau_1) > z^{ind}_{l_i}(\tau_2)$, return $\mu=1$
        \item If $z^{ind}_{l_i}(\tau_1) < z^{ind}_{l_i}(\tau_2)$, return $\mu=0$
    \end{itemize}

    \item Return $\mu=0.5$, i.e., trajectories are indistinguishable
\end{enumerate}

% In the comparison process, we need to sort pass-fail tests from most to least difficult. Here we use their historical pass rates to indicate test difficulty, where lower pass rates represent greater difficulty. Additionally, we need to sort indicative tests by their optimization progress. Here, we propose skewness as a metric for assessing optimization progress: right-skewed score distributions (most trajectories achieve low scores) reveal under-optimized metrics, while left-skewed distributions (most trajectories achieve high scores) indicate well-optimized metrics. 
% Consequently, the skewness observed in historical test results serves as a quantitative measure of optimization status.

\subsection{TdRL Algorithm}
\label{sec:algo}

\begin{algorithm}[tb]
\caption{TdRL}
\label{alg:tdrl}
\textbf{Require}: frequency of return network update $K$\\
\textbf{Require}: pass-fail tests $\{z_1^{pf}, z_2^{pf},...z_m^{pf}\}$, indicative tests $\{z_1^{ind}, z_2^{ind},...,z_n^{ind}\}$
\begin{algorithmic}[1] %[1] enables line numbers
\STATE Initial $\pi_\theta$, $R^{ind}_{\xi}$, $r_\phi$, and $\tau=\emptyset$
\STATE Initial trajectory buffer $\mathcal D=\emptyset$, replay buffer $\mathcal B=\emptyset$.
\FOR{each iteration}
\STATE Sample action $a$ from $\pi_\theta(\cdot |s)$
\STATE Exec action $a$ in environment and get $(s',done)$
\STATE Update Trajectory $\tau \leftarrow \tau \cup \{(s,a,s')\}$
\STATE Store transition $\mathcal B \leftarrow \mathcal B \cup \{(s,a,s',r_\phi(s,a), done)\}$
\IF {$done$}
\STATE Store trajectory $\mathcal D \leftarrow \mathcal D \cup \tau$, reset $\tau=\emptyset$
\ENDIF
\IF {iteration $\% K ==0$}
\FOR {each gradient step}
\STATE Sample minibatch $\{(\tau_1,\tau_2)_j\}_{j=1}^D \sim \mathcal D$
\STATE Compute $\mu$ for each pair $(\tau_1,\tau_2)$ by the lexicographic trajectory comparison approach
\STATE Optimize $\mathcal L_R^{Dis}$ in (\ref{eq:loss_return_dis}) and $\mathcal L_R^{Penalty}$ in (\ref{eq:loss_return_con})  with respect to $\xi$
\STATE Optimize $\mathcal L_r$ in (\ref{eq:loss_reward}) with respect to $\phi$
\ENDFOR
\STATE Relabel entire replay buffer $\mathcal B$ using $r_\phi$
\ENDIF
\FOR {each gradient step}
\STATE Sample random minibatch from $\mathcal B$
\STATE Update policy network $\theta$
\ENDFOR
\ENDFOR
\end{algorithmic}
\end{algorithm}

\cref{alg:tdrl} presents the detailed procedure of the TdRL algorithm. TdRL employs $\mathcal L_R^{Dis}$ to learn the return function and constrains the update step via $\mathcal L_R^{Penalty}$. However, since the gradient magnitudes of the two losses can not be directly comparable, selecting appropriate weights for balancing them is challenging. Specifically, $\mathcal L_R^{Dis}$ is a cross-entropy loss, whose gradient scales with the probability difference, whereas $\mathcal L_R^{Penalty}$ is an MSE loss, whose gradient depends on the variation in the return output. To address this issue, we propose two methods: \textit{gradient norm} (GN) and \textit{early stop} (ES). Both methods compute the gradients of $\mathcal L_R^{Dis}$ ($\nabla_\xi \mathcal L_R^{Dis}$) and $\mathcal L_R^{Penalty}$ ($\nabla_\xi \mathcal L_R^{Penalty}$) before each network update and use their sum to update the network parameters. The key distinction lies in the fact that GN rescales the MSE gradient to match the cross-entropy gradient L2-norm when the former exceeds the latter, and ES stops training if the MSE gradient L2-norm surpasses a predefined multiple ($K^{ES}$) of the cross-entropy gradient L2-norm.

\section{Experiments}
\label{sec:experiments}

We design our experiments to investigate the benefits of using test functions for representing task objectives, as well as the performance of the TdRL algorithm.
% We design experiments to investigate the following:
% \begin{itemize}
%     \item Can TdRL learn from test functions to solve tasks?
%     \item Can TdRL accomplish multi-objective tasks better than using reward functions?
%     \item What is the contribution of each proposed technique in TdRL?
%     \item Can TdRL learn novel behaviors by adding additional test functions? 
% \end{itemize}

\subsection{Setups}

We evaluate TdRL on several continuous robot control tasks from DeepMind Control Suite (DM-Control, \citeauthor{dmcontrol_2020} \citeyear{dmcontrol_2020}). We employ the Soft Actor-Critic (SAC, \citeauthor{sac_2018} \citeyear{sac_2018}) algorithm as the backbone for the implementation of TdRL. To enhance experience diversity in early training, we use unsupervised RL for warm-up. Hyperparameters for the algorithm implementation are detailed in Appendix \ref{app:hyperparameter}. In \cref{sec:algo}, we introduce two approaches—\textit{gradient norm} and \textit{early stop}—to balance the two loss terms $\mathcal L^{Dis}_R$ and $\mathcal L^{Penalty}_R$ in return function learning. For notational clarity, we refer to the TdRL variants employing these methods as TdRL-GN and TdRL-ES.

\subsection{Main Results}

We treat the environment rewards in DM-Control as the oracle rewards. To focus on the efficiency of TdRL, we only use the components of the oracle rewards in environments to construct the test functions. For instance, in the \textit{Walker-Walk} task, the oracle reward is derived from three components: torso height, torso upright, and move speed. Our test functions for this task also only use these components. Detailed test functions for each task are provided in Appendix \ref{app:tester}. More experiment results are shown in \cref{app:more_exp}

\textbf{TdRL eliminates manual objective weighting while achieving performance comparable to or better than carefully handcrafted reward functions.}
\cref{fig:dm_control_performance} demonstrates that TdRL achieves comparable or superior performance to SAC with oracle rewards in continuous robot control tasks. Since TdRL requires learning the reward function during policy optimization, its performance improvement during early training stages is slower compared to using oracle rewards. Furthermore, both TdRL-GN and TdRL-ES exhibit strong performance, indicating the effectiveness of combining cross-entropy and MSE loss in both approaches.

\textbf{TdRL can also be applied to on-policy RL algorithms.} 
Although TdRL is theoretically grounded in maximum entropy reinforcement learning, the TdRL algorithm can also be applied to on-policy RL methods. We maintain the reward learning component of TdRL while replacing its policy update algorithm with Proximal Policy Optimization (PPO, \citeauthor{ppo_2017} \citeyear{ppo_2017}). \cref{fig:dm_control_performance} demonstrates that TdRL with PPO achieves comparable performance to PPO with oracle rewards on several tasks, while exhibiting significant performance gaps on others. These results suggest TdRL's potential for integration with on-policy reinforcement learning algorithms.

\begin{figure}[tb]
\centering
\includegraphics[width=\columnwidth]{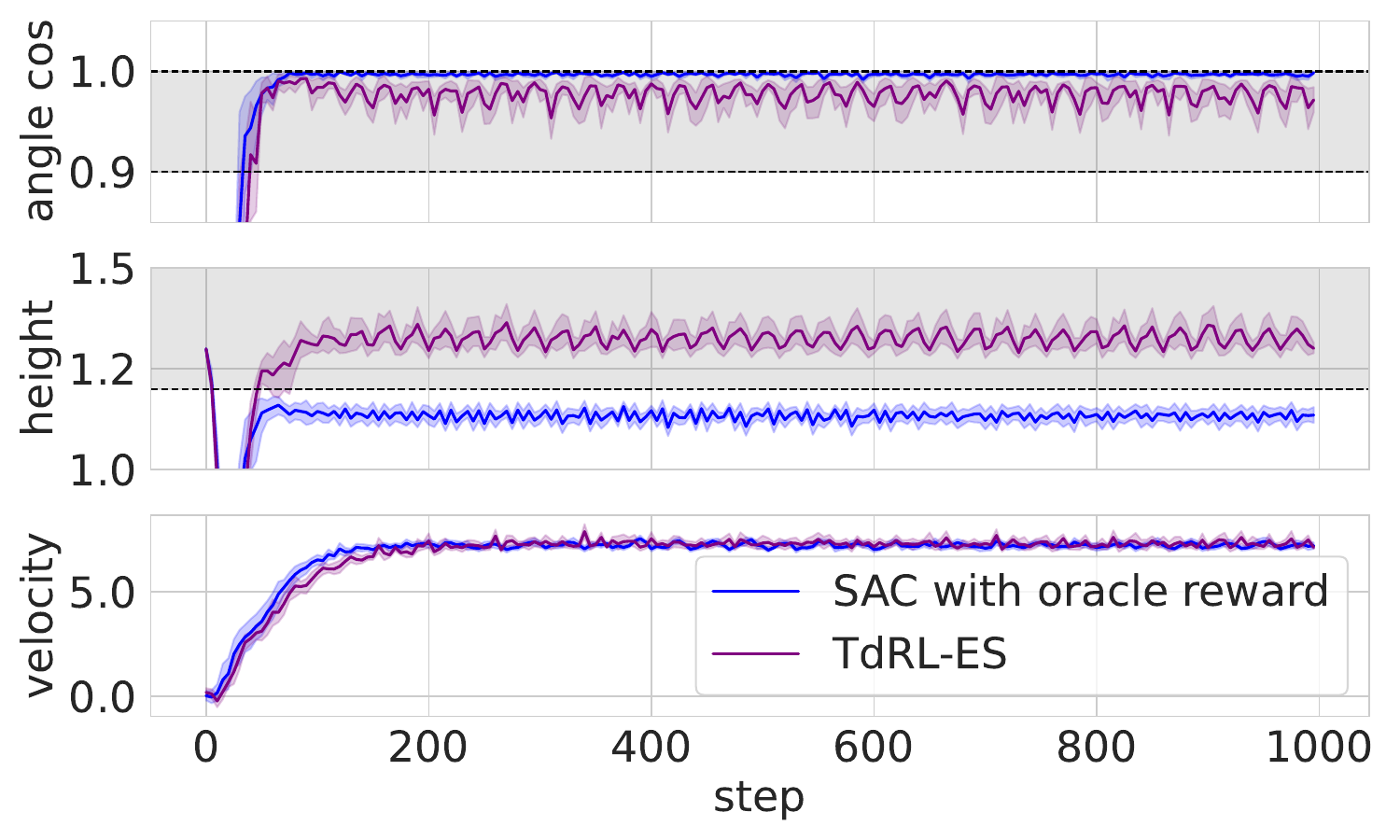}
\caption{Performance comparison in multi-objectives between SAC with oracle reward and TdRL in the \textit{Walker-Run} task. The gray shaded area represents the predefined performance threshold for each metric.}
\label{fig:fine-grained}
\end{figure}

\begin{figure*}[tb]
\centering
\begin{subfigure}[b]{0.245\textwidth}
    \centering
    \includegraphics[width=\textwidth]{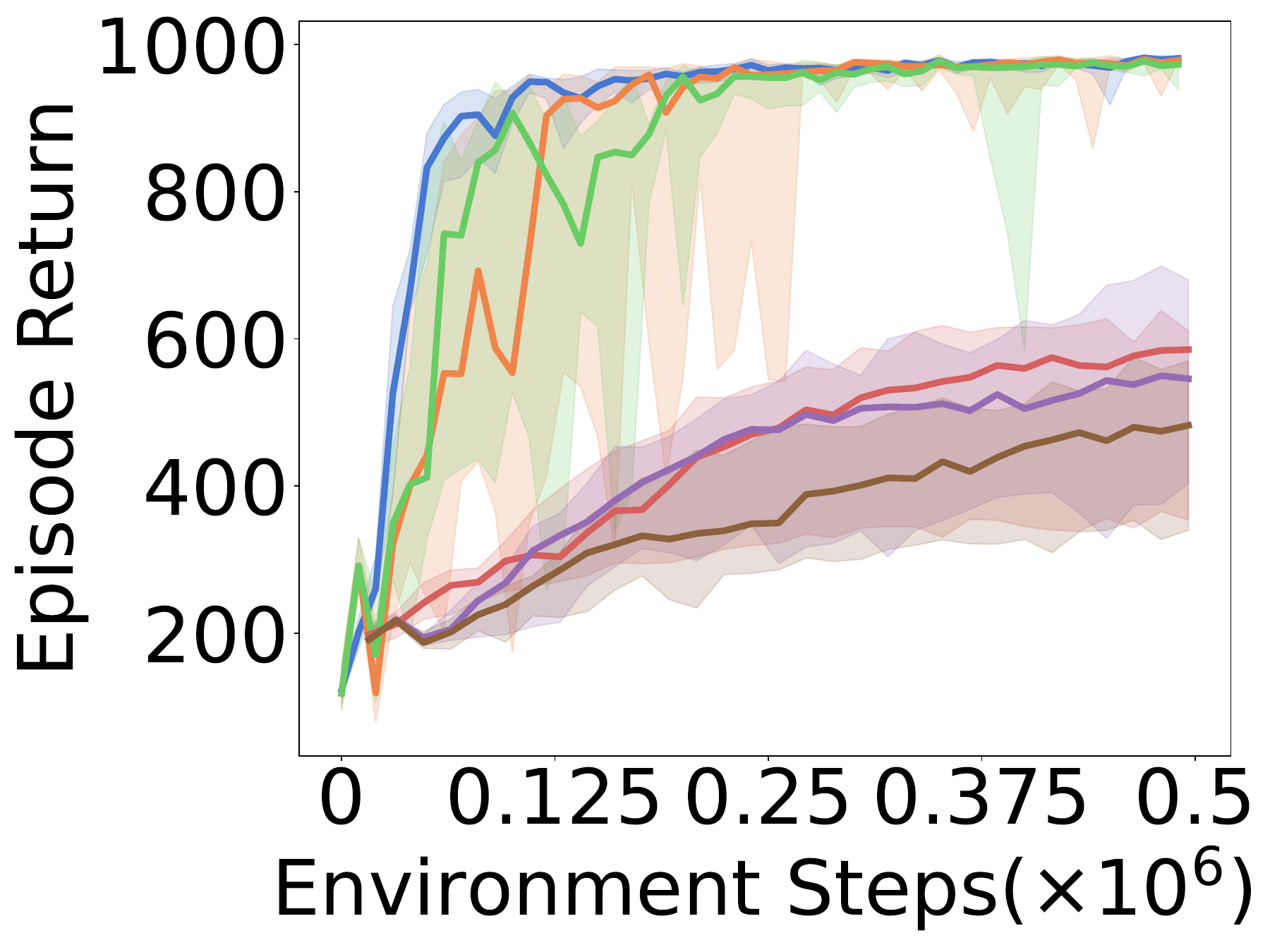}
    \caption{Walker-Stand}
\end{subfigure}
\begin{subfigure}[b]{0.245\textwidth}
    \centering
    \includegraphics[width=\textwidth]{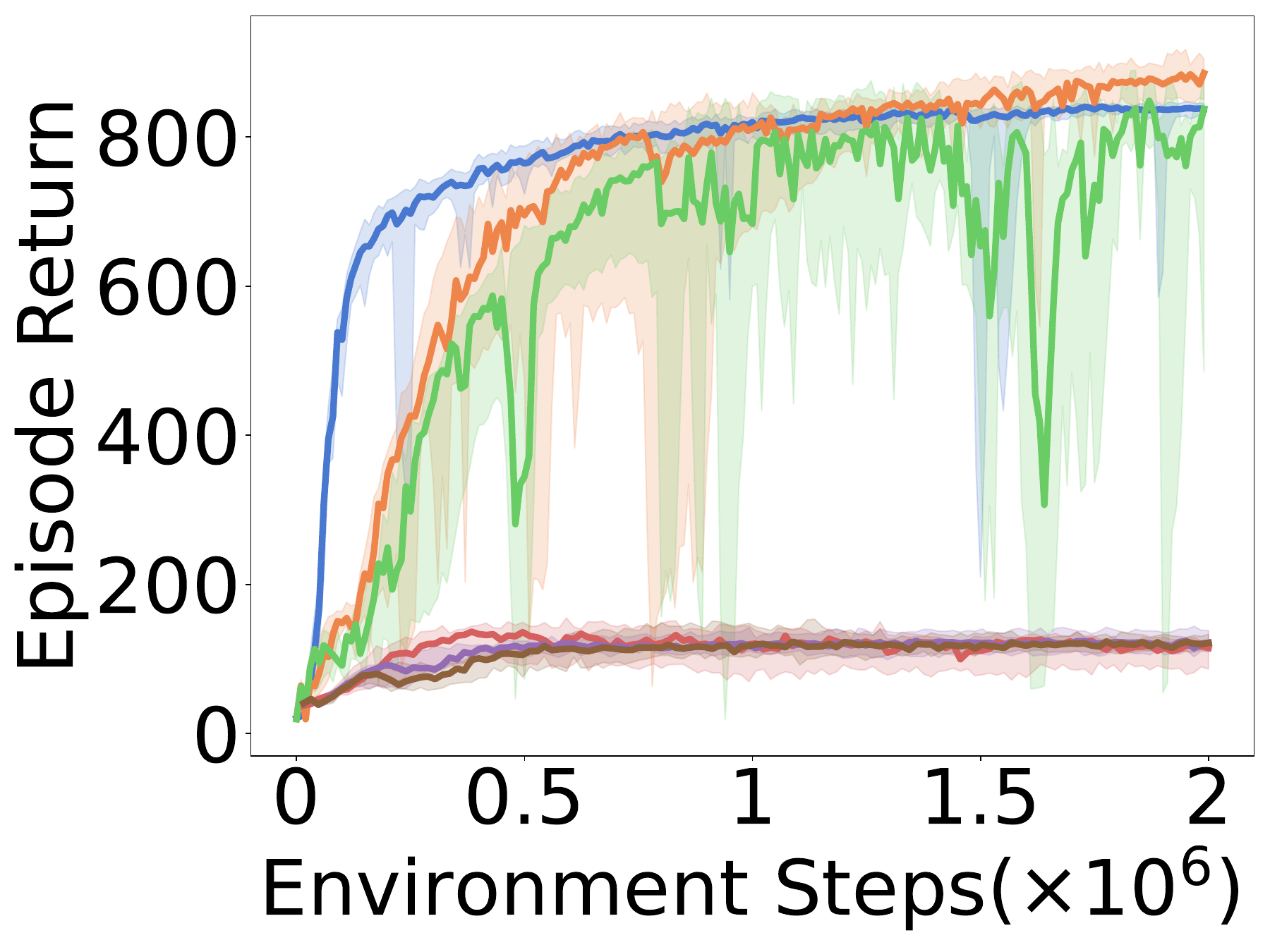}
    \caption{Walker-Run}
\end{subfigure}
\begin{subfigure}[b]{0.245\textwidth}
    \centering
    \includegraphics[width=\textwidth]{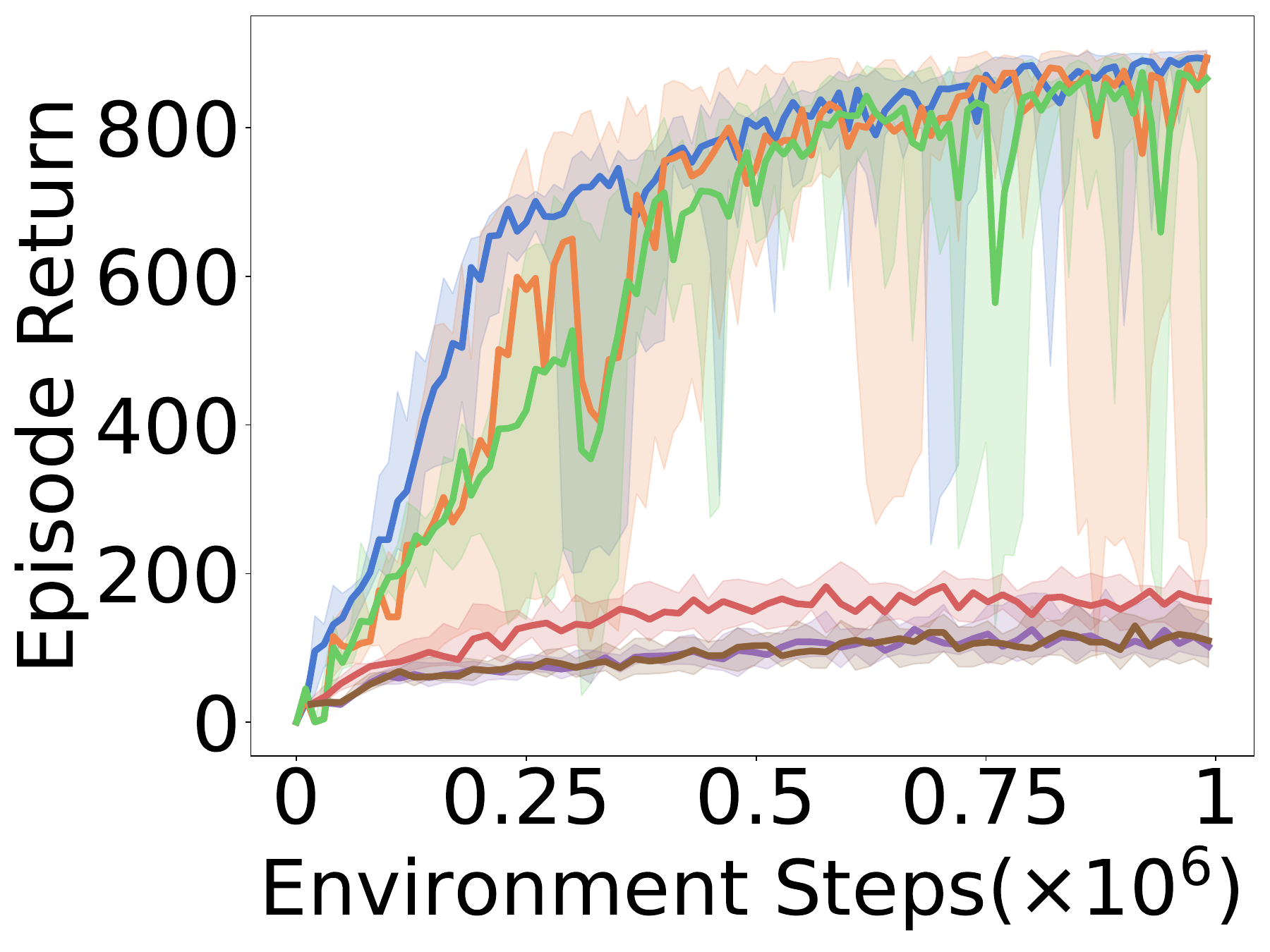}
    \caption{Cheetah-Run}
\end{subfigure}
\begin{subfigure}[b]{0.245\textwidth}
    \centering
    \includegraphics[width=\textwidth]{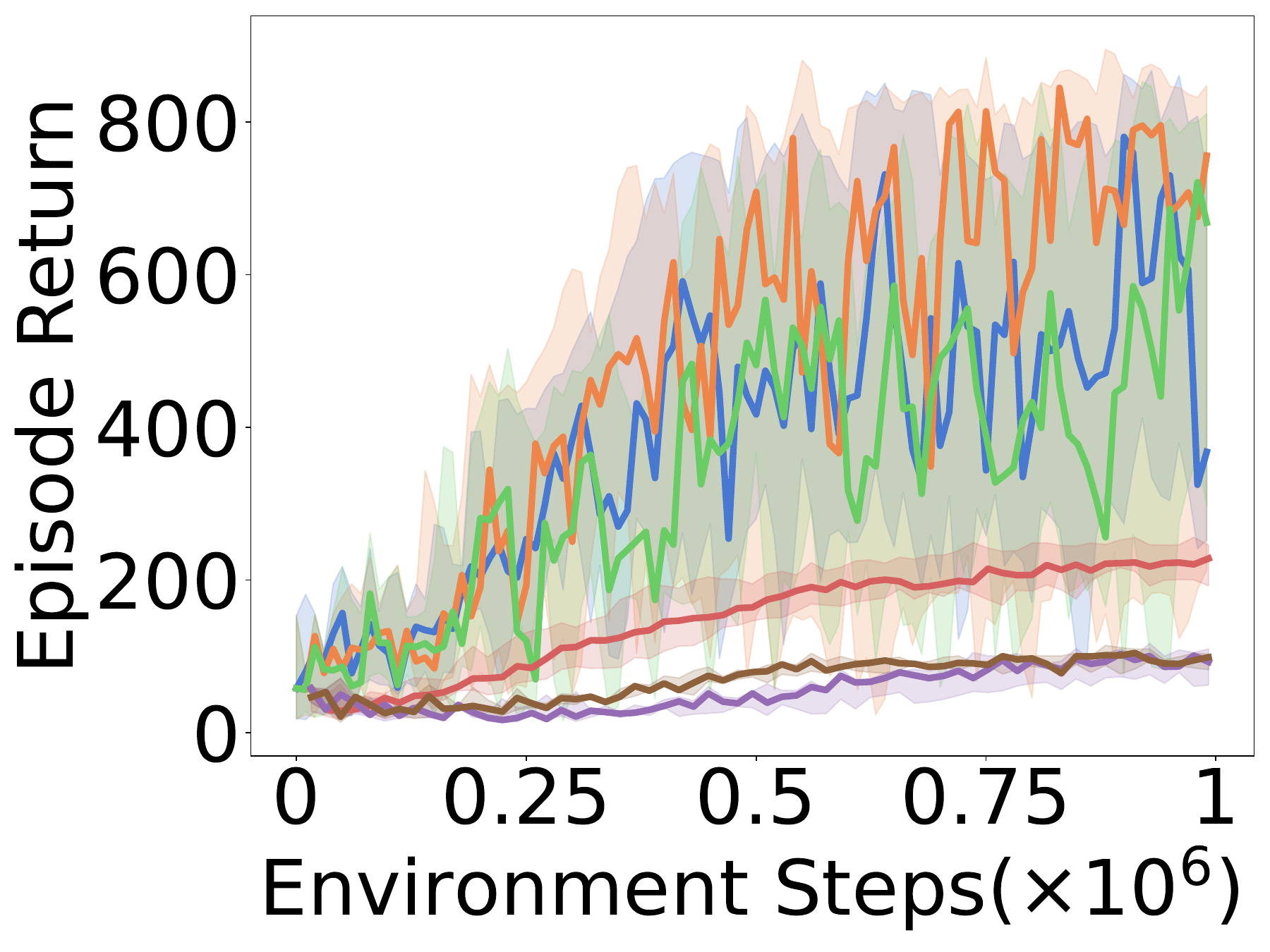}
    \caption{Quadruped-Run}
\end{subfigure}
\begin{subfigure}[b]{\textwidth}
    \centering
    \includegraphics[width=\textwidth]{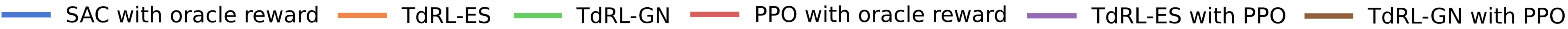}
\end{subfigure}

% TODO
\caption{Performance comparison of algorithms on DM-Control tasks. Each algorithm runs with 10 different random seeds. Following \citeauthor{statistic_2021} \shortcite{statistic_2021}, the solid lines represent the interquartile mean (IQM) of episode returns while the shaded areas indicate 95\% confidence intervals.}
\label{fig:dm_control_performance}
\end{figure*}

\textbf{TdRL achieves a satisficing solution across multiple objectives rather than an optimal solution in a single metric.}
\cref{fig:fine-grained} shows that the performance comparison in multi-objectives between SAC with oracle reward and TdRL in the \textit{Walker-Run} task. 
In the \textit{Walker-Run} task, there are three primary objectives. The cosine of the torso angle should be within [0.9, 1], ensuring the robot's upper body remains upright. The torso height should exceed 1.2, ensuring the robot maintains a standing posture. The velocity in the x-axis should reach 8. The reward function for this task in DM-Control is defined as $((3*stand + upright) / 4) * (5*move + 1) / 6$, where $upright$, $stand$, and $move$ respectively represent the rewards of uprightness, standing height, and move speed of the robot. \cref{fig:fine-grained} shows that despite $stand$ having a higher weight than $upright$ in the reward function, the trained policy achieves the desired uprightness but fails to maintain sufficient standing height. 
% This suggests a trade-off between standing height and speed optimization: the agent likely reduces its height to maximize forward velocity. Since uprightness has negligible influence on speed, it remains consistently satisfied. 

In TdRL, however, for each of the three metrics - uprightness, stand height, and move speed - we constructed both pass-fail tests and indicative tests (the test functions are detailed in the Appendix \ref{app:tester}). Notably, we do not need to preset the weights to combine them; instead, the pass-fail tests define the passing conditions for each metric. \cref{fig:fine-grained} demonstrates that the TdRL policy fulfills all task metrics more effectively. While TdRL does not always achieve the optimum in individual metrics (e.g., upper-body uprightness), it consistently attains a satisficing solution across multiple objectives. In the \textit{Walker-Run} task, TdRL matches the speed of the SAC using oracle reward while satisfying both uprightness and stand height requirements.

\textbf{TdRL focuses on the performance of trajectories rather than the quality of state-action pairs.} \cref{fig:fine-grained} shows that the TdRL-trained policy underperforms the policy trained by SAC with oracle rewards across all metrics in stability. This discrepancy likely arises because TdRL operates at the trajectory level rather than evaluating state-action pairs. Such trajectory-level evaluation is more intuitive for designers, avoiding task objectives that overly emphasize state quality \cite{booth_2023}, which may lead to reward hacking \cite{amodei_2016}. Furthermore, TdRL can also enhance stability for specific metrics through introducing additional test functions (e.g., the variance of the robot's standing height).

\subsection{Ablation Study}

\textbf{Impact of reward learning approaches}
\cref{fig:ablation}(a) illustrates the impact of different reward learning approaches in the TdRL algorithm. 
TdRL with no penalty denotes the variant without applying $\mathcal L^{Penalty}_R$ to constrain the trajectory return function learning, while TdRL with direct reward learning refers to directly learning the reward function from trajectory comparisons rather than first learning the return function followed by decomposition. 
Experimental results demonstrate that both omitting the penalty term and directly learning the reward function lead to training instability and degraded policy performance. Specifically, the absence of a penalty causes uncontrolled growth in return values during learning, potentially inducing numerical instability. When directly learning the reward function, $tanh$ activation is typically applied in the reward network's output layer to bound its outputs, which requires continuous rescaling during later training stages to accommodate reward learning, ultimately resulting in training instability and performance deterioration.

\begin{figure}[htbp]
\centering
\begin{subfigure}[b]{0.49\columnwidth}
    \centering
    \includegraphics[width=\textwidth]{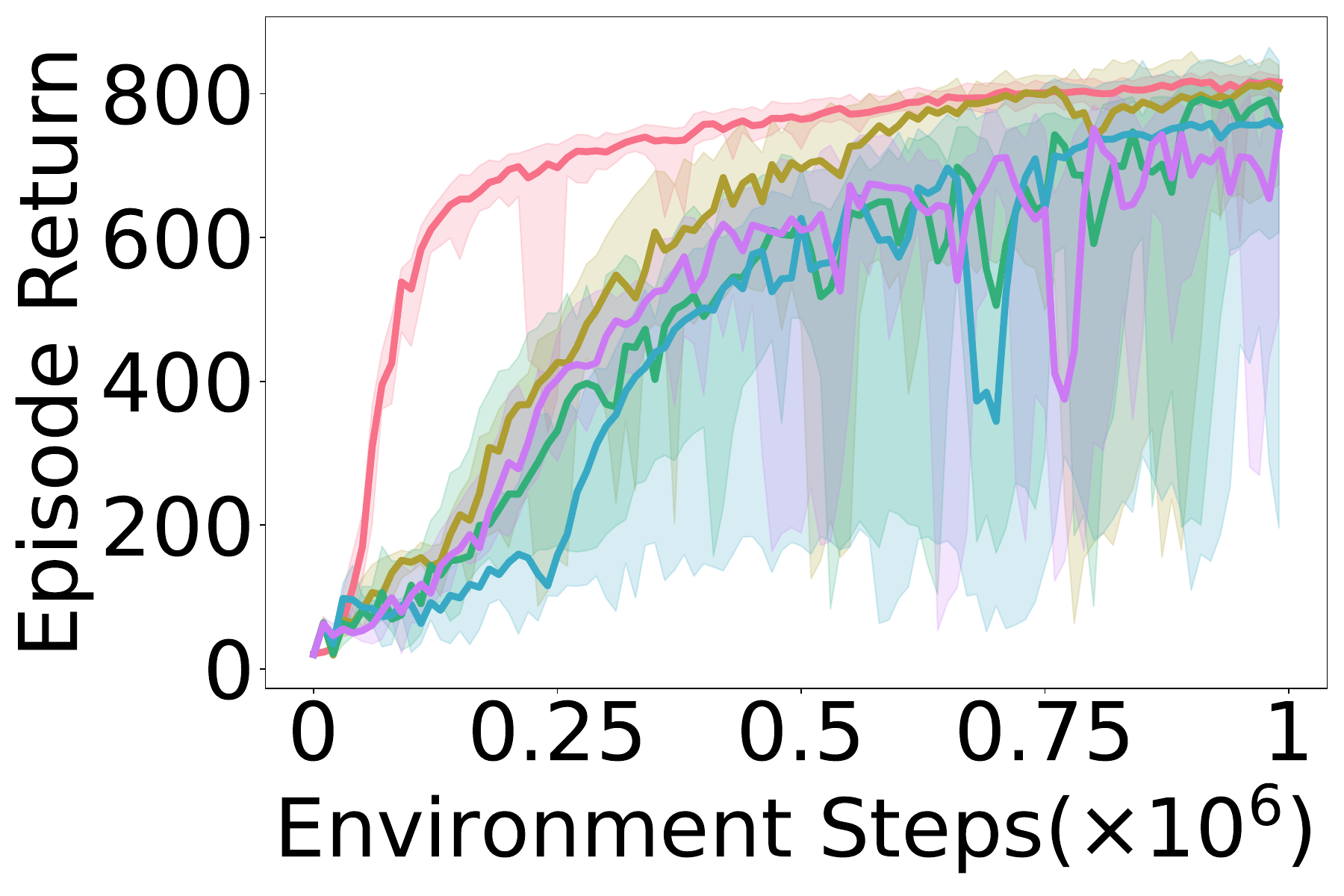}
    \includegraphics[width=\textwidth]{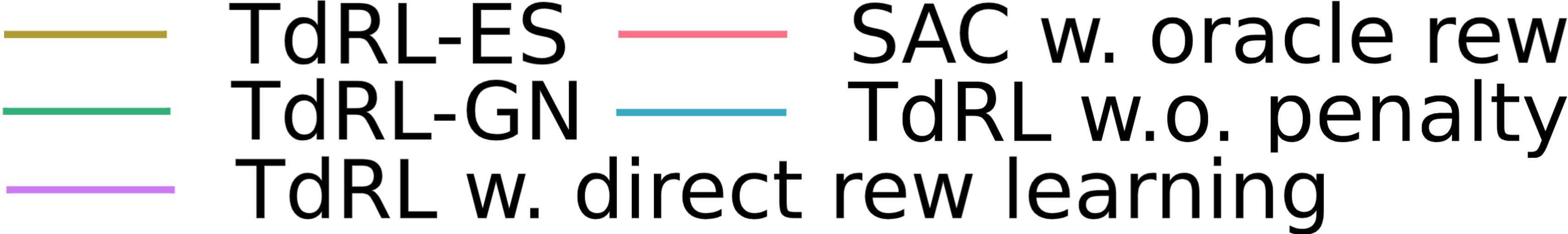}
    \caption{}
    \label{fig:reward_decompose}
\end{subfigure}
\begin{subfigure}[b]{0.49\columnwidth}
    \centering
    \includegraphics[width=\textwidth]{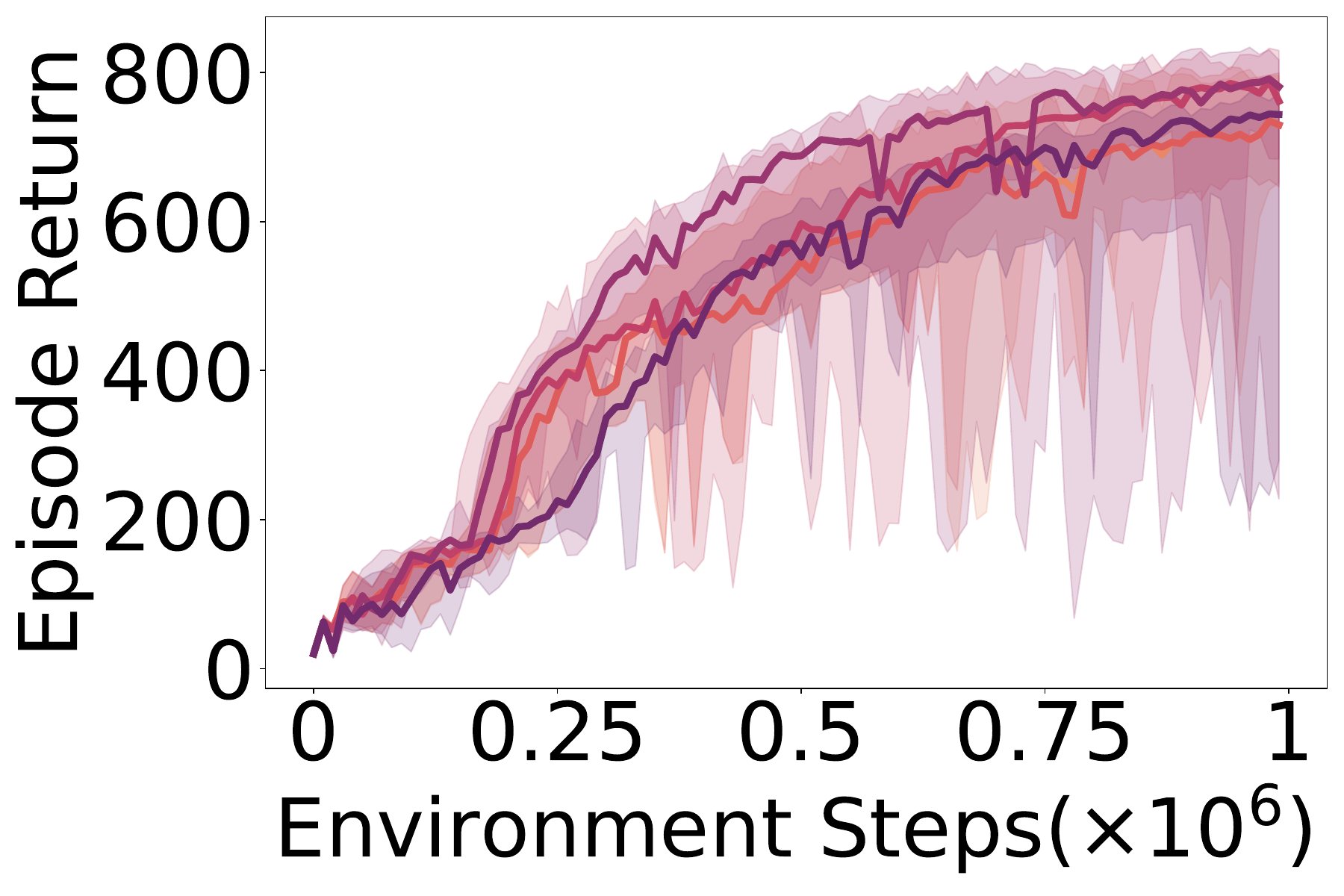}
    \includegraphics[width=0.8\textwidth]{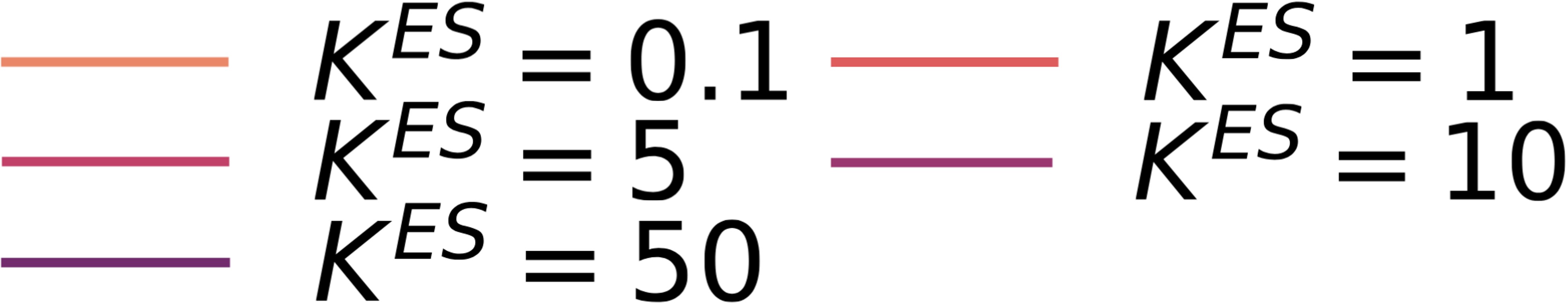}
    \caption{}
    \label{fig:hyperparameters}
\end{subfigure}
\caption{\textbf{Left:} the performance of TdRL with different reward learning methods. \textbf{Right:} the performance of TdRL-ES with varying values of multiple $K^{ES}$.}
\label{fig:ablation}
\end{figure}
 
\textbf{Hyperparameters choice} 
TdRL-ES stops training when the L2-norm of $\mathcal{L}_R^{Penalty}$'s gradient exceeds a predefined multiple ($K^{ES}$) of $\mathcal{L}_R^{Dis}$'s gradient L2-norm. The scaling factor $K^{ES}$ is a tunable parameter requiring configuration. We investigated how different multiple $K^{ES}$ values affect algorithm performance. \cref{fig:ablation}(b) indicates that both excessively large and small multiple $K^{ES}$ values degrade performance. Based on experimental results, we recommend setting multiple $K^{ES}$ to 10 as a general guideline.

\section{Conclusion}
\label{sec:conclusion}

To tackle the challenges in reward design for reinforcement learning, this paper introduces a test-driven reinforcement learning (TdRL) framework. Instead of relying on a scalar reward function, TdRL represents task objectives using multiple test functions. The test functions are defined on trajectories rather than state-action pairs, which is more intuitive for designers. Besides, each test function only needs to represent a single objective, and designers do not need to consider their weights, greatly reducing the complexity of the reinforcement learning task design.

We propose sufficient conditions for trajectory return functions to guarantee policy convergence to the optimal policy set.
Then, we introduce a lexicographic trajectory comparison approach for return learning. Furthermore, we present an implementation of TdRL, and the experimental results show that TdRL matches or outperforms reward-based methods in task training, with greater design simplicity and inherent support for multi-objective optimization. TdRL offers a viable approach to tackling reward design challenges in reinforcement learning.

\setcounter{secnumdepth}{0}
\section{Acknowledgements}

We thank all anonymous reviewers for their constructive feedback, which helped improve our paper. We also thank Mr. Yuxin Cheng from the University of Hong Kong for his helpful discussions.

\bibliography{reference}

@misc{ppo_2017,
  title={{Proximal} {Policy} {Optimization} {Algorithms}},
  author={John Schulman and Filip Wolski and Prafulla Dhariwal and Alec Radford and Oleg Klimov},
  journal={ArXiv},
  year={2017},
  volume={abs/1707.06347},
  url={https://api.semanticscholar.org/CorpusID:28695052}
}

@article{sac_2018,
  author       = {Tuomas Haarnoja and
                  Aurick Zhou and
                  Kristian Hartikainen and
                  George Tucker and
                  Sehoon Ha and
                  Jie Tan and
                  Vikash Kumar and
                  Henry Zhu and
                  Abhishek Gupta and
                  Pieter Abbeel and
                  Sergey Levine},
  title        = {Soft {Actor-Critic} {Algorithms} and {Applications}},
  journal      = {CoRR},
  volume       = {abs/1812.05905},
  year         = {2018},
  url          = {http://arxiv.org/abs/1812.05905},
  eprinttype    = {arXiv},
  eprint       = {1812.05905},
  bibsource    = {dblp computer science bibliography, https://dblp.org}
}

@inproceedings{statistic_2021,
author = {Agarwal, Rishabh and Schwarzer, Max and Castro, Pablo Samuel and Courville, Aaron and Bellemare, Marc G.},
title = {Deep reinforcement learning at the edge of the statistical precipice},
year = {2021},
isbn = {9781713845393},
publisher = {Curran Associates Inc.},
address = {Red Hook, NY, USA},
booktitle = {Proceedings of the 35th International Conference on Neural Information Processing Systems},
articleno = {2244},
numpages = {17},
series = {NIPS '21}
}

@incollection{lexicographic_2020,
  author    = {Özgür Şimşek},
  title     = {Lexicographic Decision Rule},
  booktitle = {Oxford Research Encyclopedia of Politics},
  year      = {2020},
  month     = oct,
  publisher = {Oxford University Press},
  doi       = {10.1093/acrefore/9780190228637.013.908},
  url       = {https://oxfordre.com/politics/view/10.1093/acrefore/9780190228637.001.0001/acrefore-9780190228637-e-908}
}

@article{satisficing_1947,
	title = {Administrative behavior; a study of decision-making processes in administrative organization.},
	journal = {Administrative behavior; a study of decision-making processes in administrative organization.},
	author = {Simon, H. A.},
	year = {1947},
	note = {Place: Oxford,  England
Publisher: Macmillan},
	pages = {xvi, 259--xvi, 259},
}

@article{llm_survey_2025,
  author={Cao, Yuji and Zhao, Huan and Cheng, Yuheng and Shu, Ting and Chen, Yue and Liu, Guolong and Liang, Gaoqi and Zhao, Junhua and Yan, Jinyue and Li, Yun},
  journal={IEEE Transactions on Neural Networks and Learning Systems}, 
  title={Survey on Large Language Model-Enhanced Reinforcement Learning: Concept, Taxonomy, and Methods}, 
  year={2025},
  volume={36},
  number={6},
  pages={9737-9757},
  doi={10.1109/TNNLS.2024.3497992}}

@article{irl_survey_2021,
title = {A survey of inverse reinforcement learning: Challenges, methods and progress},
journal = {Artificial Intelligence},
volume = {297},
pages = {103500},
year = {2021},
issn = {0004-3702},
doi = {https://doi.org/10.1016/j.artint.2021.103500},
url = {https://www.sciencedirect.com/science/article/pii/S0004370221000515},
author = {Saurabh Arora and Prashant Doshi}
}

@article{pbrl_open_problems_2023,
  author={Stephen Casper and Xander Davies and Claudia Shi and Thomas Krendl Gilbert and Jérémy Scheurer and Javier Rando and Rachel Freedman and Tomasz Korbak and David Lindner and Pedro Freire and Tony Tong Wang and Samuel Marks and Charbel-Raphaël Ségerie and Micah Carroll and Andi Peng and Phillip J. K. Christoffersen and Mehul Damani and Stewart Slocum and Usman Anwar and Anand Siththaranjan and Max Nadeau and Eric J. Michaud and Jacob Pfau and Dmitrii Krasheninnikov and Xin Chen and Lauro Langosco and Peter Hase and Erdem Biyik and Anca D. Dragan and David Krueger and Dorsa Sadigh and Dylan Hadfield-Menell},
  title={Open Problems and Fundamental Limitations of Reinforcement Learning from Human Feedback},
  year={2023},
  cdate={1672531200000},
  journal={Trans. Mach. Learn. Res.},
  volume={2023},
  url={https://openreview.net/forum?id=bx24KpJ4Eb}
}

@inproceedings{pbrl_llm_2022,
author = {Ouyang, Long and Wu, Jeff and Jiang, Xu and Almeida, Diogo and Wainwright, Carroll L. and Mishkin, Pamela and Zhang, Chong and Agarwal, Sandhini and Slama, Katarina and Ray, Alex and Schulman, John and Hilton, Jacob and Kelton, Fraser and Miller, Luke and Simens, Maddie and Askell, Amanda and Welinder, Peter and Christiano, Paul and Leike, Jan and Lowe, Ryan},
title = {Training language models to follow instructions with human feedback},
year = {2022},
isbn = {9781713871088},
publisher = {Curran Associates Inc.},
address = {Red Hook, NY, USA},
booktitle = {Proceedings of the 36th International Conference on Neural Information Processing Systems},
articleno = {2011},
numpages = {15},
location = {New Orleans, LA, USA},
series = {NIPS '22}
}

@inproceedings{maxent_irl_2008,
author = {Ziebart, Brian D. and Maas, Andrew and Bagnell, J. Andrew and Dey, Anind K.},
title = {Maximum entropy inverse reinforcement learning},
year = {2008},
isbn = {9781577353683},
publisher = {AAAI Press},
booktitle = {Proceedings of the 23rd National Conference on Artificial Intelligence - Volume 3},
pages = {1433–1438},
numpages = {6},
location = {Chicago, Illinois},
series = {AAAI'08}
}

@article{dmcontrol_2020,
         title = {dm\_control: Software and tasks for continuous control},
         journal = {Software Impacts},
         volume = {6},
         pages = {100022},
         year = {2020},
         issn = {2665-9638},
         doi = {https://doi.org/10.1016/j.simpa.2020.100022},
         url = {https://www.sciencedirect.com/science/article/pii/S2665963820300099},
         author = {Saran Tunyasuvunakool and Alistair Muldal and Yotam Doron and
                   Siqi Liu and Steven Bohez and Josh Merel and Tom Erez and
                   Timothy Lillicrap and Nicolas Heess and Yuval Tassa},
}

@misc{pbrl_2017,
      title={Deep reinforcement learning from human preferences}, 
      author={Paul Christiano and Jan Leike and Tom B. Brown and Miljan Martic and Shane Legg and Dario Amodei},
      year={2017},
      eprint={1706.03741},
      archivePrefix={arXiv},
      primaryClass={stat.ML},
      url={https://arxiv.org/abs/1706.03741}, 
}

@InProceedings{pebble_2021,
  title = 	 {PEBBLE: Feedback-Efficient Interactive Reinforcement Learning via Relabeling Experience and Unsupervised Pre-training},
  author =       {Lee, Kimin and Smith, Laura M and Abbeel, Pieter},
  booktitle = 	 {Proceedings of the 38th International Conference on Machine Learning},
  pages = 	 {6152--6163},
  year = 	 {2021},
  editor = 	 {Meila, Marina and Zhang, Tong},
  volume = 	 {139},
  series = 	 {Proceedings of Machine Learning Research},
  month = 	 {18--24 Jul},
  publisher =    {PMLR},
  url = 	 {https://proceedings.mlr.press/v139/lee21i.html}
}

@inproceedings{surf_2022,
title={{SURF}: Semi-supervised Reward Learning with Data Augmentation for Feedback-efficient Preference-based Reinforcement Learning},
author={Jongjin Park and Younggyo Seo and Jinwoo Shin and Honglak Lee and Pieter Abbeel and Kimin Lee},
booktitle={International Conference on Learning Representations},
year={2022},
url={https://openreview.net/forum?id=TfhfZLQ2EJO}
}

@misc{heron_2024,
      title={Deep Reinforcement Learning from Hierarchical Preference Design}, 
      author={Alexander Bukharin and Yixiao Li and Pengcheng He and Tuo Zhao},
      year={2024},
      eprint={2309.02632},
      archivePrefix={arXiv},
      primaryClass={cs.LG},
      url={https://arxiv.org/abs/2309.02632}, 
}

@inproceedings{rewardlm_2023,
title={Reward Design with Language Models},
author={Minae Kwon and Sang Michael Xie and Kalesha Bullard and Dorsa Sadigh},
booktitle={The Eleventh International Conference on Learning Representations },
year={2023},
url={https://openreview.net/forum?id=10uNUgI5Kl}
}

@inproceedings{l2r_2023,
title={Language to Rewards for Robotic Skill Synthesis},
author={Wenhao Yu and Nimrod Gileadi and Chuyuan Fu and Sean Kirmani and Kuang-Huei Lee and Montserrat Gonzalez Arenas and Hao-Tien Lewis Chiang and Tom Erez and Leonard Hasenclever and Jan Humplik and brian ichter and Ted Xiao and Peng Xu and Andy Zeng and Tingnan Zhang and Nicolas Heess and Dorsa Sadigh and Jie Tan and Yuval Tassa and Fei Xia},
booktitle={7th Annual Conference on Robot Learning},
year={2023},
url={https://openreview.net/forum?id=SgTPdyehXMA}
}

@InProceedings{ellm_2023,
  title = 	 {Guiding Pretraining in Reinforcement Learning with Large Language Models},
  author =       {Du, Yuqing and Watkins, Olivia and Wang, Zihan and Colas, C\'{e}dric and Darrell, Trevor and Abbeel, Pieter and Gupta, Abhishek and Andreas, Jacob},
  booktitle = 	 {Proceedings of the 40th International Conference on Machine Learning},
  pages = 	 {8657--8677},
  year = 	 {2023},
  editor = 	 {Krause, Andreas and Brunskill, Emma and Cho, Kyunghyun and Engelhardt, Barbara and Sabato, Sivan and Scarlett, Jonathan},
  volume = 	 {202},
  series = 	 {Proceedings of Machine Learning Research},
  month = 	 {23--29 Jul},
  publisher =    {PMLR},
  url = 	 {https://proceedings.mlr.press/v202/du23f.html}
}

@book{tdd_2002,
author = {Beck},
title = {Test Driven Development: By Example},
year = {2002},
isbn = {0321146530},
publisher = {Addison-Wesley Longman Publishing Co., Inc.},
address = {USA}
}

@article{tddreward_2022,
title = {Test-Driven Reward Function for Reinforcement Learning: A Contribution towards Applicable Machine Learning Algorithms for Production Systems},
journal = {Procedia CIRP},
volume = {112},
pages = {103-108},
year = {2022},
note = {15th CIRP Conference on Intelligent Computation in ManufacturingEngineering, 14-16 July 2021},
issn = {2212-8271},
doi = {https://doi.org/10.1016/j.procir.2022.09.043},
url = {https://www.sciencedirect.com/science/article/pii/S2212827122011970},
author = {Florian Jaensch and Karl Kübler and Elmar Schwarz and Alexander Verl}
}

@INPROCEEDINGS{tddirl_2024,
  author={Fischer, Johannes and Werling, Moritz and Lauer, Martin and Stiller, Christoph},
  booktitle={2024 IEEE Intelligent Vehicles Symposium (IV)}, 
  title={Test-Driven Inverse Reinforcement Learning Using Scenario-Based Testing}, 
  year={2024},
  volume={},
  number={},
  pages={827-834},
  doi={10.1109/IV55156.2024.10588652}
}

@inproceedings{testimportance_2024,
title={Test Where Decisions Matter: Importance-driven Testing for Deep Reinforcement Learning},
author={Stefan Pranger and Hana Chockler and Martin Tappler and Bettina K{\"o}nighofer},
booktitle={The Thirty-eighth Annual Conference on Neural Information Processing Systems},
year={2024},
url={https://openreview.net/forum?id=TwrnhZfD6a}
}

@article{silver_2021,
title = {Reward is enough},
journal = {Artificial Intelligence},
volume = {299},
pages = {103535},
year = {2021},
issn = {0004-3702},
doi = {https://doi.org/10.1016/j.artint.2021.103535},
url = {https://www.sciencedirect.com/science/article/pii/S0004370221000862},
author = {David Silver and Satinder Singh and Doina Precup and Richard S. Sutton}
}

@article{silver_2025,
  title={Welcome to the era of experience},
  author={Silver, David and Sutton, Richard S},
  journal={Google AI},
  volume={1},
  year={2025}
}

@book{sutton_rl_2018,
	address = {Cambridge, MA, USA},
	title = {Reinforcement {Learning}: {An} {Introduction}},
	isbn = {978-0-262-03924-6},
	shorttitle = {Reinforcement {Learning}},
	publisher = {A Bradford Book},
	author = {Sutton, Richard S. and Barto, Andrew G.},
	month = oct,
	year = {2018},
}

@article{knox_misdesign_2023,
  title = {Reward (Mis)design for autonomous driving},
  journal = {Artificial Intelligence},
  volume = {316},
  pages = {103829},
  year = {2023},
  issn = {0004-3702},
  doi = {https://doi.org/10.1016/j.artint.2022.103829},
  url = {https://www.sciencedirect.com/science/article/pii/S0004370222001692},
  author = {W. Bradley Knox and Alessandro Allievi and Holger Banzhaf and Felix Schmitt and Peter Stone}
}

@article{booth_2023, 
  title={The Perils of Trial-and-Error Reward Design: Misdesign through Overfitting and Invalid Task Specifications}, 
  volume={37}, 
  url={https://ojs.aaai.org/index.php/AAAI/article/view/25733}, DOI={10.1609/aaai.v37i5.25733}, 
  abstractNote={In reinforcement learning (RL), a reward function that aligns exactly with a task’s true performance metric is often necessarily sparse. For example, a true task metric might encode a reward of 1 upon success and 0 otherwise. The sparsity of these true task metrics can make them hard to learn from, so in practice they are often replaced with alternative dense reward functions. These dense reward functions are typically designed by experts through an ad hoc process of trial and error. In this process, experts manually search for a reward function that improves performance with respect to the task metric while also enabling an RL algorithm to learn faster. This process raises the question of whether the same reward function is optimal for all algorithms, i.e., whether the reward function can be overfit to a particular algorithm. In this paper, we study the consequences of this wide yet unexamined practice of trial-and-error reward design. We first conduct computational experiments that confirm that reward functions can be overfit to learning algorithms and their hyperparameters. We then conduct a controlled observation study which emulates expert practitioners’ typical experiences of reward design, in which we similarly find evidence of reward function overfitting. We also find that experts’ typical approach to reward design---of adopting a myopic strategy and weighing the relative goodness of each state-action pair---leads to misdesign through invalid task specifications, since RL algorithms use cumulative reward rather than rewards for individual state-action pairs as an optimization target. Code, data: github.com/serenabooth/reward-design-perils}, 
  number={5}, 
  journal={Proceedings of the AAAI Conference on Artificial Intelligence}, 
  author={Booth, Serena and Knox, W. Bradley and Shah, Julie and Niekum, Scott and Stone, Peter and Allievi, Alessandro}, 
  year={2023}, 
  month={Jun.}, 
  pages={5920-5929} 
}

@article{rajagopal_2023,
  title={What’s Next if Reward is Enough? Insights for AGI from Animal Reinforcement Learning},
  author={Rajagopal, Shreya},
  journal={Journal of Artificial General Intelligence},
  volume={14},
  number={1},
  pages={15--40},
  year={2023},
  publisher={De Gruyter Poland}
}

@misc{amodei_2016,
      title={Concrete Problems in AI Safety}, 
      author={Dario Amodei and Chris Olah and Jacob Steinhardt and Paul Christiano and John Schulman and Dan Mané},
      year={2016},
      eprint={1606.06565},
      archivePrefix={arXiv},
      primaryClass={cs.AI},
      url={https://arxiv.org/abs/1606.06565}, 
}

@inproceedings{peter_2023,
author = {Vamplew, Peter and Smith, Benjamin J. and K\"{a}llstr\"{o}m, Johan and Ramos, Gabriel and R\u{a}dulescu, Roxana and Roijers, Diederik M. and Hayes, Conor F. and Hentz, Friedrik and Mannion, Patrick and Libin, Pieter J.K. and Dazeley, Richard and Foale, Cameron},
title = {Scalar Reward is Not Enough},
year = {2023},
isbn = {9781450394321},
publisher = {International Foundation for Autonomous Agents and Multiagent Systems},
address = {Richland, SC},
booktitle = {Proceedings of the 2023 International Conference on Autonomous Agents and Multiagent Systems},
pages = {839–841},
numpages = {3},
location = {London, United Kingdom},
series = {AAMAS '23}
}

@article{bradley_1952,
  title={Rank analysis of incomplete block designs: I. the method of paired comparisons},
  author={Bradley, Ralph Allan and Terry, Milton E},
  journal={Biometrika},
  volume={39},
  number={3/4},
  pages={324--345},
  year={1952},
  publisher={JSTOR}
}

@book{villani_2008,
  title={Optimal transport: old and new},
  author={Villani, C{\'e}dric and others},
  volume={338},
  year={2008},
  publisher={Springer}
}

@article{Juechems_2019,
title = {Where Does Value Come From?},
journal = {Trends in Cognitive Sciences},
volume = {23},
number = {10},
pages = {836-850},
year = {2019},
issn = {1364-6613},
doi = {https://doi.org/10.1016/j.tics.2019.07.012},
url = {https://www.sciencedirect.com/science/article/pii/S1364661319302001},
author = {Keno Juechems and Christopher Summerfield}
}

@article{gigerenzer_1996,
  title={Reasoning the fast and frugal way: models of bounded rationality.},
  author={Gigerenzer, Gerd and Goldstein, Daniel G},
  journal={Psychological review},
  volume={103},
  number={4},
  pages={650},
  year={1996},
  publisher={American Psychological Association}
}

@article{kk_2013,
author = {Katsikopoulos, Konstantinos V.},
title = {Why Do Simple Heuristics Perform Well in Choices with Binary Attributes?},
journal = {Decision Analysis},
volume = {10},
number = {4},
pages = {327-340},
year = {2013},
doi = {10.1287/deca.2013.0281},
URL = { 
        https://doi.org/10.1287/deca.2013.0281
},
eprint = { 
        https://doi.org/10.1287/deca.2013.0281
}
}

@book{kk_2021,
    author = {Katsikopoulos, Konstantinos V. and \c{S}im\c{s}ek, {\"O}zg{\"u}r and Buckmann, Marcus and Gigerenzer, Gerd},
    title = {Classification in the Wild: The Science and Art of Transparent Decision Making},
    publisher = {The MIT Press},
    year = {2021},
    month = {02},
    isbn = {9780262363228},
    doi = {10.7551/mitpress/11790.001.0001},
    url = {https://doi.org/10.7551/mitpress/11790.001.0001},
}

\appendix
\setcounter{secnumdepth}{2}
\section{Limitation}

Our work has several limitations. First, our theoretical analysis assumes the maximum entropy policy optimization framework. While empirical results show that TdRL works well with other RL algorithms (e.g., PPO), it lacks theoretical guarantees beyond maximum entropy optimization. Second, our trajectory comparison relies on a lexicographic heuristic. Future studies should investigate methods with stronger theoretical foundations. Moreover, the test functions employed in this study are manually designed. Future work could explore integrating large language models to assist test function design.

\section{TdRL Algorithm Analysis}
\label{app:tdrl_alalysis}

\textbf{Computational Efficiency Analysis}. For each trajectory, its test results are deterministic and fixed during training, so caching these results can reduce computational overhead. However, in trajectory comparison, the order of test functions may change as training progresses, necessitating re-comparing the trajectories whenever the reward network is updated. The primary computational overhead of TdRL arises from these comparisons. Given $M$ reward network updates, each involving $N$ trajectory pairs, the total computational overhead is $\mathcal O(M N)$. Moreover, since trajectory comparisons are mutually independent, they can be executed in parallel to improve computational efficiency.

\textbf{Does TdRL eliminate the burden of reward design}? Firstly, the transformation of problems often leads to more diverse solutions. PbRL turns reward engineering into the collection of human preference labels, while IRL turns it into the collection of demonstration data. Difficulties in transformation can be addressed from a new perspective.
Secondly, recent literature \cite{rajagopal_2023,knox_misdesign_2023,booth_2023} point out that in the practice of reinforcement learning reward design, the problems with reward design are usually caused by designers focusing on the state and neglecting its impact on the trajectory. TdRL directly focuses on the quality of the trajectory during design tests, which is more intuitive.
Additionally, reward design often requires considering trade-offs between multiple optimization objectives, which is challenging and usually requires a significant amount of time and resources for tuning. TdRL only needs to set the required satisfaction thresholds for each optimization objective, without designing weights between them; the trade-off between multiple optimization objectives is automatically completed during the reward learning in the TdRL algorithm.
In summary, we believe that TdRL, compared to reward engineering, reduces the difficulty of reward design and offers better support for multi-objective optimization. Moreover, this shift also provides a new perspective for addressing future challenges in reward design.

\begin{table}[htbp]
    \centering
    \caption{Comparison of TdRL and other reward design/learning methods.}
    \label{tab:comparision}
    \begin{tabular}{>{\centering\arraybackslash}m{1.5cm}|c|c|c|>{\centering\arraybackslash}m{2cm}}
        \hline
        & \textbf{TdRL} & \textbf{PbRL} & \textbf{IRL} & \textbf{Reward Engineering} \\
        \hline
        No preference label required & $\checkmark$ & $\times$ & $\checkmark$ & $\checkmark$ \\
        \hline
        No demonstration required & $\checkmark$ & $\checkmark$ & $\times$ & $\checkmark$ \\
        \hline
        Trajectory-level evaluation & $\checkmark$ & $\checkmark$ & / & $\times$ \\
        \hline
    \end{tabular}
    
\end{table}

\textbf{Compare with other reward design/learning methods}. \cref{tab:comparision} compares TdRL (Test-driven Reinforcement Learning), PbRL (Preference-based Reinforcement Learning), IRL (Inverse Reinforcement Learning), and Reward Engineering. TdRL is independent of human preference labels or expert demonstrations, and its core learning process relies on trajectory-level evaluation. In contrast, PbRL requires human-provided preference labels and also employs trajectory evaluation. IRL  relies heavily on expert demonstrations to infer reward functions. Finally, Reward Engineering, which involves the manual design of reward functions, requires neither preference labels nor demonstrations; however, it focuses on state evaluation rather than trajectory evaluation, often resulting in suboptimal reward functions during the design process.

\section{Hyperparameter}
\label{app:hyperparameter}

\begin{table}[tb]
\centering
\caption{Hyperparameter Settings}
\label{tab:hyperparameters}
\begin{tabular}{l|cccc}
    \hline
    ~ & \textbf{SAC} & \textbf{TdRL} & \textbf{PPO} & \makecell{\textbf{PPO}\\\textbf{TdRL}} \\
    \hline
    discount factor($\gamma$) & \multicolumn{4}{c}{0.99} \\
    critic hidden dim & \multicolumn{4}{c}{1024} \\
    actor hidden dim & \multicolumn{4}{c}{1024} \\
    critic depth & \multicolumn{4}{c}{2} \\
    actor depth & \multicolumn{4}{c}{2} \\
    optimizer & \multicolumn{4}{c}{Adam} \\
    adam $\beta_1$ & \multicolumn{4}{c}{0.9} \\
    adam $\beta_2$ & \multicolumn{4}{c}{0.999} \\
    actor lr & \multicolumn{2}{c}{0.0005} & \multicolumn{2}{c}{0.0003} \\
    critic lr & \multicolumn{2}{c}{0.0005} & \multicolumn{2}{c}{0.0003} \\
    alpha lr & \multicolumn{2}{c}{0.0001} & \multicolumn{2}{c}{/} \\
    batch size & \multicolumn{2}{c}{1024} & \multicolumn{2}{c}{64} \\
    critic $\tau$ & \multicolumn{2}{c}{0.005} & \multicolumn{2}{c}{/} \\
    unsupervised steps & / & 9000 & / & 9000 \\
    trajectory max num & / & 100 & / & 100 \\
    segment size & / & 50 & / & 50 \\
    rew lr & / & 0.0003 & / & 0.0003 \\
    ret lr & / & 0.0003 & / & 0.0003 \\
    rew ensemble & / & 3 & / & 3 \\
    ret ensemble & / & 3 & / & 3 \\
    rew batch size & / & 128 & / & 128 \\
    ret batch size & / & 128 & / & 128 \\
    rew update num & / & 50 & / & 50 \\
    ret update num & / & 50 & / & 50 \\
    ret change penalty & / & 0.1 & / & 0.1 \\
    rew update interval & / & 5000 & / & 5000 \\
    ret update interval & / & 5000 & / & 5000 \\
    es multiple & / & 10 & / & 10 \\
    \hline
\end{tabular}
\end{table}

\cref{tab:hyperparameters} uses some standard reinforcement learning abbreviations for conciseness: Learning rates are abbreviated as "lr" (e.g., actor lr), "rew" and "ret" shorten reward and return-related terms (e.g., rew lr is reward learning rate). These conventions align with common RL literature to ensure clarity while saving space.

\section{More Experimental Results}
\label{app:more_exp}

\begin{table*}[htbp]
\centering
\caption{Performance Comparison of Different Methods}
\label{tab:performance}
\begin{tabular}{c|c|c|c|c|c}
\hline
\textbf{env}           & \textbf{metric}   & \textbf{target} & \textbf{\makecell{SAC \\ (Oracle Reward)}} & \textbf{TdRL-ES} & \textbf{TdRL-GN} \\ \hline
\multirow{2}{*}{Walker-Stand}  & upright      &    $[0.9,1]$    & $0.94\pm 0.03$                    & $\mathbf{0.95\pm0.02}$        & $0.94\pm0.03$        \\ 
                      &stand height &    $[1.2,+\infty]$    & $1.23\pm0.02$                    & $\mathbf{1.26\pm 0.02}$        & $1.25\pm 0.03$        \\ \hline
\multirow{3}{*}{Walker-Run}    & upright      &    $[0.9,1]$    & $\mathbf{0.97\pm0.02}$                    & $0.94\pm0.02$        & $0.93\pm0.02$        \\ 
                      &stand height &    $[1.2,+\infty]$    & \textcolor{blue}{$1.12\pm 0.01$}                    & $1.30\pm0.01$        & $\mathbf{1.26\pm0.02}$        \\
                      & speed        &    $[8,+\infty]$    & \textcolor{blue}{$\mathbf{6.80\pm0.08}$}                    & $\textcolor{blue}{6.78\pm0.07}$        & $\textcolor{blue}{6.45\pm0.14}$        \\ \hline
Cheetah-Run   & speed        &    $[10,+\infty]$    & $\textcolor{blue}{9.44\pm 0.19}$                    & $\textcolor{blue}{9.56\pm0.37}$        & \textcolor{blue}{$\mathbf{9.88\pm0.08}$}        \\ \hline
\multirow{2}{*}{Quadruped-Run} & upright      &    $[0.9, 1]$    & $\mathbf{0.94\pm0.03}$                    & $0.93\pm0.04$        & $0.91\pm0.13$        \\ 
                      & speed        &    $[5,+\infty]$    & \textcolor{blue}{$3.95\pm0.09$ }                   &  \textcolor{blue}{$4.00\pm0.13$}        & \textcolor{blue}{$\mathbf{4.56\pm0.49}$}        \\ \hline
\end{tabular}
\end{table*}

\cref{tab:performance} shows the performance of policies trained by SAC with oracle rewards and TdRL algorithms in multiple DM-Control tasks. The table details the metrics of the policy with the highest episode reward during training, where each policy is tested ten times, and the average performance and standard deviation of these ten tests are recorded. Black values in the table indicate that the performance meets the target, blue values indicate that the performance does not meet the target, and bold values represent the highest value among all algorithms for that metric.

\textbf{Preformance Comparison}. The experimental results in \cref{tab:performance} show that in multiple tasks, TdRL can achieve or exceed the performance of SAC with oracle rewards, and it can demonstrate better performance when SAC with oracle rewards fails to guarantee certain metrics meet the target. Additionally, for some metrics in certain tasks (e.g., speed in Walker-Run), neither the policies trained by SAC with oracle rewards nor TdRL can achieve the task objectives, but TdRL generally exhibits better performance on these metrics. This could be because TdRL optimizes the unmet metrics as much as possible after ensuring other metrics meet the target, whereas SAC with oracle rewards may continue improving other metrics to further improve cumulative rewards, which may lead to reward hacking.

\begin{figure}[htbp]
\centering
\includegraphics[width=\columnwidth]{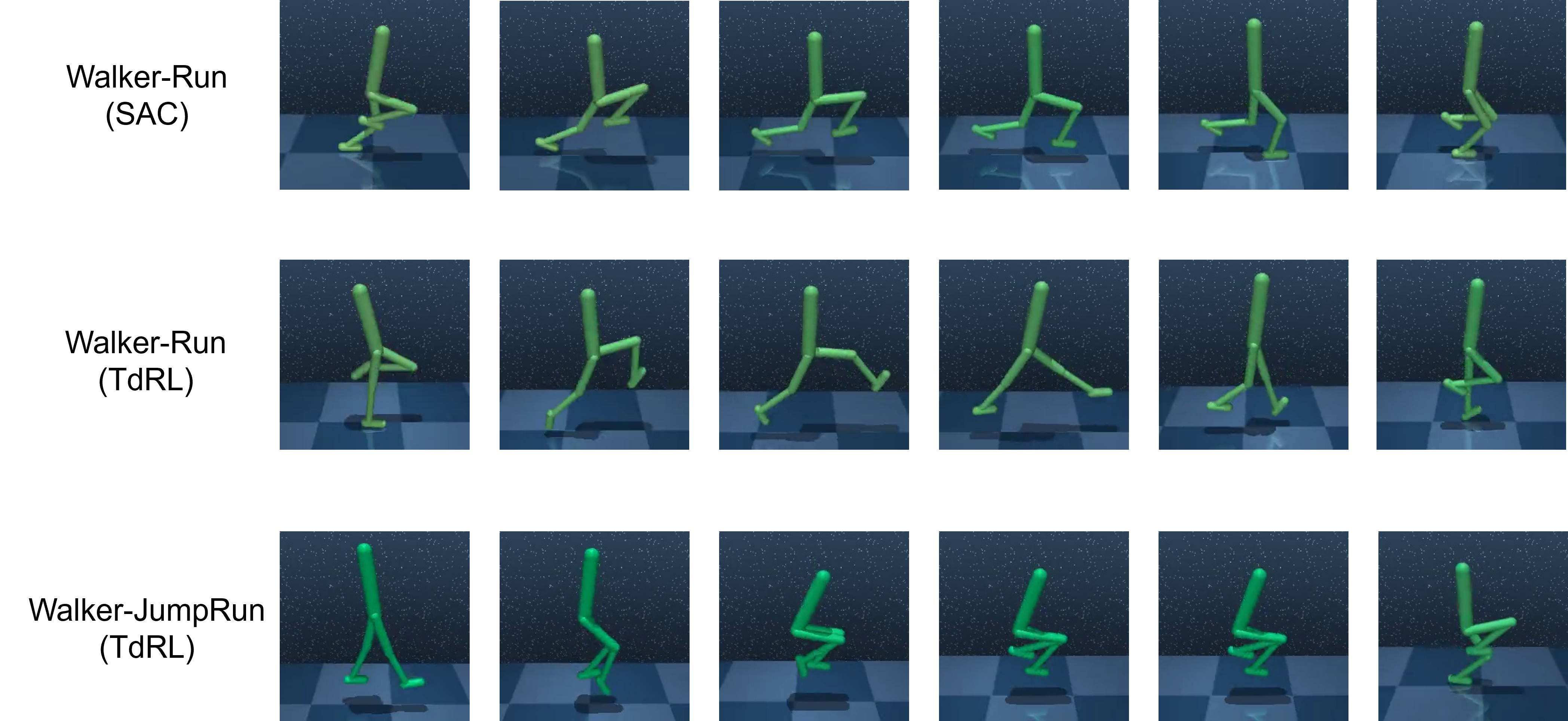}
\caption{The frames of policies trained by SAC with oracle reward and TdRL in the Walker-Run task, as well as the performance of the policy trained by TdRL on a new task, \textit{Walker-JumpRun}. \textit{Walker-JumpRun} adds a test of the maximum torsor height in trajectory based on the Walker-Run task, as detailed in Appendix \ref{app:tester}.}
\label{fig:run_frame}
\end{figure}

\textbf{Novel Task}. \cref{fig:run_frame} presents frames from the \textit{Walker-Run} task, comparing policies trained by SAC with oracle reward and TdRL. The TdRL-trained policy consistently raises the legs higher during each running stride than the policy trained by oracle reward, thereby maintaining the desired torso height throughout the trajectory. Furthermore, we evaluated TdRL on a new task, \textit{Walker-JumpRun}, which extends the \textit{Walker-Run} task by introducing an additional test for maximum torso height. As shown in \cref{fig:run_frame}, to satisfy the maximum height requirement, the torso must perform repeated jumps during locomotion. Simultaneously, to maximize forward speed, the torso retracts its legs mid-jump, enabling greater forward displacement per jump. These results demonstrate that TdRL can more readily accomplish complex tasks compared to manual reward function design by incorporating additional test functions.

\begin{table*}[htbp]
\centering
\caption{Performance metrics for different speed thresholds and torso angle threshold}
\label{tab:speed_angle_sensitive}
\begin{tabular}{c|c|c|c|c}
\hline
\textbf{speed\textbackslash torso angle} & \textbf{0.9} & \textbf{0.95} & \textbf{0.98} & \textbf{0.99} \\
\hline
2 & 0.94/1.31/2.05 & 0.97/1.30/2.02 & 0.98/1.30/3.87 & 0.99/1.26/1.98 \\
4 & 0.95/1.26/4.01 & 0.97/1.24/3.98 & 0.98/1.27/3.96 & 0.99/1.24/3.93 \\
6 & 0.95/1.30/5.86 & 0.96/1.32/5.81 & 0.98/1.29/5.64 & 0.99/1.26/5.18 \\
8 & 0.93/1.28/7.34 & 0.94/1.32/7.27 & 0.98/1.29/6.31 & 0.99/1.25/5.55 \\
\hline
\end{tabular}
\end{table*}

\textbf{Sensitivity Analysis}. \cref{tab:speed_angle_sensitive} shows the performance metrics tested on the \textit{walker-run} environment by systematically varying the target speed and the torso angle threshold. The three metrics recorded in each cell of the results table are: torso angle / average height / average speed. It shows that a clear performance trade-off exists. Specifically, higher target speeds (e.g., 8 m/s) become increasingly difficult to achieve as the torso angle constraint is tightened (from 0.9 to 0.99 rad), evidenced by the significant drop in the achieved speed metric (from 7.34 to 5.55). Conversely, at lower target speeds (e.g., 2 m/s), the agent can maintain a very close torso angle and consistent height across all thresholds. This indicates that the policy's ability to precisely control its torso angle while maintaining stability is compromised when pushed to its maximum velocity limits. Moreover, the overall experimental results demonstrate that TdRL tends to prioritize achieving the more attainable threshold metrics when thresholds vary, while making efforts to satisfy other challenging performance metrics.

\section{Proof of \cref{eq:maxent}}
\label{app:proof_eq:maxent}

When policy $\pi_1$ is optimized via maximum entropy reinforcement learning with respect to the trajectory return function $R(\tau)=\sum_{(s,a)\in\tau} r(s,a)$, the resulting policy $\pi_2$ follows:
$$
    \pi_2(\tau) = \frac{1}{Z} \pi_1(\tau) \exp\left( \frac{1}{\alpha} R(\tau) \right).
$$

\begin{proof}
    The optimization objective of Maximum Entropy Reinforcement Learning (MERL) is:
$$
        \max_{\pi} \mathbb{E}_{\tau \sim \pi} \left[ \sum_{t=0}^T r(s_t, a_t) + \alpha H(\pi(\cdot|s_t)) \right],
$$
where $\alpha > 0$ is a temperature parameter that balances reward maximization and entropy maximization, and the entropy of the policy at state $s_t$ is defined as
$$
H(\pi(\cdot|s_t)) = - \mathbb{E}_{a_t \sim \pi(\cdot|s_t)} \left[ \log \pi(a_t|s_t) \right].
$$

Since the trajectory distribution depends on the policy's action distribution, and the environment dynamics are fixed, the total entropy contribution over the trajectory can be expressed as

\begin{align}
    \sum_{t=0}^T H(\pi(\cdot|s_t)) &= - \mathbb{E}_{\tau \sim \pi} \left[ \sum_{t=0}^T \log \pi(a_t|s_t) \right] \notag \\ 
    &= - \mathbb{E}_{\tau \sim \pi} \left[ \log \pi(\tau) \right] + \text{const}, \notag
\end{align}

where the constant comes from the environment dynamics and initial state distribution, which do not depend on $\pi$.

Therefore, the maximum entropy objective can be equivalently written as
$$
\max_{\pi} \mathbb{E}_{\tau \sim \pi} \left[ R(\tau) - \alpha \log \pi(\tau) \right].
$$

The objective above can be viewed as maximizing the expected return regularized by the negative entropy of the trajectory distribution. Equivalently, this corresponds to minimizing the Kullback-Leibler (KL) divergence between the policy-induced trajectory distribution $\pi(\tau)$ and a reward-weighted distribution.
Formally, define a base trajectory distribution $\pi_1(\tau)$ induced by some prior policy or uniform policy:
$$
\pi_1(\tau) = p(s_0) \prod_{t=0}^T \pi_1(a_t|s_t) p(s_{t+1}|s_t, a_t).
$$

The maximum entropy objective implies the optimal trajectory distribution $\pi_2(\tau)$ satisfies:
$$
\pi_2 = \arg\min_{\pi} D_{\mathrm{KL}}\left( \pi(\tau) \, \Big\| \, \frac{1}{Z} \pi_1(\tau) \exp\left( \frac{1}{\alpha} R(\tau) \right) \right),
$$
where $Z$ is the partition function ensuring normalization:
$$
Z = \int \pi_1(\tau) \exp\left( \frac{1}{\alpha} R(\tau) \right) d\tau.
$$

Solving this minimization yields the optimal distribution:
$$
\pi_2(\tau) = \frac{1}{Z} \pi_1(\tau) \exp\left( \frac{1}{\alpha} R(\tau) \right).
$$
\end{proof}

\section{Proof of \cref{theorem:R_pi}}
\label{app:proof_R_pi}

\begin{lemma}
    \label{lemma:ratio_monotonically}
    Let the trajectory return function $R(\tau)$ be monotonically non-increasing with respect to the distance between a trajectory $\tau$ and the optimal trajectory set $\tilde{\mathcal{T}}$:
$$
d(\tau_1, \tilde{\mathcal{T}}) \leq d(\tau_2, \tilde{\mathcal{T}}) \implies R(\tau_1) \geq R(\tau_2).
$$
Let $P_{\pi_1}(\tau)$ denote the trajectory distribution of the baseline policy, and the maximum entropy optimization yields the trajectory distribution of the improved policy as
$$
P_{\pi_2}(\tau) = \frac{1}{Z} P_{\pi_1}(\tau) \exp\left( \frac{1}{\alpha} R(\tau) \right),
$$
where $Z$ is the partition function. By mapping trajectories to their distances $\rho = d(\tau, \tilde{\mathcal{T}})$, we define the corresponding probability density functions $P_{\pi_i}(\rho)$. Then, the likelihood ratio function
$$
r(\rho) := {P_{\pi_2}(\rho)}/{P_{\pi_1}(\rho)}
$$
is a monotonically non-increasing function of the distance $\rho$
\end{lemma}

\begin{proof}
    Let us define the mapping:
$$\Phi: \tau \mapsto \rho = d(\tau, \tilde{\mathcal{T}}),$$
which projects the trajectory space onto the non-negative real axis. Assume the trajectory space is equipped with a measure $\mu$, and the baseline trajectory distribution $P_{\pi_1}(\tau)$ represents a probability density function with respect to $\mu$.
From probability theory, the marginal probability density is defined as:
$$P_{\pi_1}(\rho) = \int_{\tau \in \Phi^{-1}(\rho)} P_{\pi_1}(\tau) \, d\mu(\tau).$$
We define the conditional probability density (the distribution of trajectories $\tau$ at fixed distance $\rho$) as:
$$P_{\pi_1}(\tau | \rho) = \frac{P_{\pi_1}(\tau)}{P_{\pi_1}(\rho)}, \quad \forall \tau \in \Phi^{-1}(\rho).$$
By the definition of the maximum entropy policy distribution:
$$P_{\pi_2}(\tau) = \frac{1}{Z} P_{\pi_1}(\tau) \exp\left(\frac{1}{\alpha} R(\tau)\right),$$
the corresponding marginal probability density becomes:
\begin{align}
P_{\pi_2}(\rho) &= \int_{\Phi^{-1}(\rho)} P_{\pi_2}(\tau) d\mu(\tau) \notag \\
&= \frac{1}{Z} \int_{\Phi^{-1}(\rho)} P_{\pi_1}(\tau) \exp\left(\frac{1}{\alpha} R(\tau)\right) d\mu(\tau). \notag
\end{align}
Expressed in conditional probability form:
\begin{align}
    &P_{\pi_2}(\rho) \notag \\
    &= \frac{1}{Z} P_{\pi_1}(\rho) \int_{\Phi^{-1}(\rho)} P_{\pi_1}(\tau | \rho) \exp\left(\frac{1}{\alpha} R(\tau)\right) d\mu(\tau). \notag
\end{align}
We define:
$$r(\rho) := \frac{P_{\pi_2}(\rho)}{P_{\pi_1}(\rho)} = \frac{1}{Z} \mathbb{E}_{\tau \sim P_{\pi_1}(\cdot|\rho)} \left[ \exp\left(\frac{1}{\alpha} R(\tau)\right) \right].$$
Given the monotonic relationship:
$$d(\tau_1, \tilde{\mathcal{T}}) \leq d(\tau_2, \tilde{\mathcal{T}}) \implies R(\tau_1) \geq R(\tau_2),$$
the reward function $R(\tau)$ is a monotonically non-increasing function of distance $\rho$. Consequently, for any $\rho\geq 0$, the range of $R(\tau)$ over the trajectory set $\Phi^{-1}(\rho)$ forms a bounded interval:
$$R(\Phi^{-1}(\rho)) := \{ R(\tau) \mid \tau \in \Phi^{-1}(\rho) \} \subseteq [m(\rho), M(\rho)],$$
where:
$$m(\rho) := \inf_{\tau \in \Phi^{-1}(\rho)} R(\tau), \quad M(\rho) := \sup_{\tau \in \Phi^{-1}(\rho)} R(\tau).$$
Moreover, $M(\rho)$ is a monotonically non-increasing function of $\rho$:
$$\rho_1 \leq \rho_2 \implies M(\rho_1) \geq M(\rho_2).$$
Thus, all trajectory rewards within $\Phi^{-1}(\rho)$ lie within a bounded interval, and as $\rho$ increases, the upper bound of $R(\tau)$ decreases.
We now employ Jensen's inequality to analyze the monotonicity of $r(\rho)$:

Since the exponential function is convex:

\begin{align}
    r(\rho) &= \frac{1}{Z} \mathbb{E}_{\tau|\rho} \left[ \exp\left(\frac{1}{\alpha} R(\tau)\right) \right] \notag \\
    &\geq \frac{1}{Z} \exp\left( \frac{1}{\alpha} \mathbb{E}_{\tau|\rho}[R(\tau)] \right). \notag
\end{align}

Defining the conditional expected reward:

$$\bar{R}(\rho) := \mathbb{E}_{\tau \sim P_{\pi_1}(\cdot|\rho)} [R(\tau)],$$
which inherits the monotonically non-increasing property from $R(\tau)$.
This yields the lower bound:
$$r(\rho) \geq \frac{1}{Z} \exp\left( \frac{1}{\alpha} \bar{R}(\rho) \right),$$
where the right-hand side is monotonically non-increasing.
Furthermore, the upper bound is given by:

\begin{align}
    r(\rho) &\leq \frac{1}{Z} \exp\left( \frac{1}{\alpha} \max_{\tau \in \Phi^{-1}(\rho)} R(\tau) \right) \notag \\
    &= \frac{1}{Z} \exp\left( \frac{1}{\alpha} M(\rho) \right), \notag
\end{align}
which is also monotonically non-increasing.
Therefore, $r(\rho)$ is sandwiched between monotonically non-increasing functions, and under typical conditions where the conditional distribution becomes increasingly concentrated, $r(\rho)$ itself converges to being strictly monotonically non-increasing.
\end{proof}

% The proof of \cref{lemma:ratio_monotonically} is detailed in Appendix \ref{app:proof_lemma:ratio_monotonically}.
\cref{lemma:ratio_monotonically} shows that if a return function $R(\tau)$ is monotonically non-increasing with the distance of a trajectory $\tau$ from the optimal trajectory set, then the policy optimized under maximum entropy reinforcement learning will assign higher probabilities to trajectories nearer to the optimal trajectory set.

\begin{lemma}
    \label{lemma:distance_geq}
    Consider a trajectory space $(\mathcal{T}, d)$ equipped with a distance metric that maps a trajectory $\tau$ to its distance from the optimal trajectory set $\tilde{\mathcal{T}}$:
$$
\rho := d(\tau, \tilde{\mathcal{T}}) \in [0, +\infty).
$$
The trajectory distributions induced by policies $\pi_1$ and $\pi_2$ are represented by probability density functions $P_{\pi_1}(\rho)$ and $P_{\pi_2}(\rho)$, respectively, where $\rho \in \mathbb{R}_{\geq 0}$. Suppose there exists a non-increasing function $r(\rho)$ such that for all $\rho \geq 0$,
$$
\frac{P_{\pi_2}(\rho)}{P_{\pi_1}(\rho)} = r(\rho),
$$
meaning that $\pi_2$ assigns higher relative probability density to trajectories with smaller $\rho$ (i.e., closer to $\tilde{\mathcal{T}}$). The optimal policy set $\tilde{\Pi}$ consists of policies that generate trajectory distributions concentrated at $\rho = 0$ (i.e., a Dirac delta distribution $\delta_0$).
Then, the following inequality holds:
$$
d(\pi_1, \tilde{\Pi}) \geq d(\pi_2, \tilde{\Pi}).
$$
\end{lemma}

\begin{proof}
    Since all policies in $\tilde{\Pi}$ induce the Dirac delta trajectory distribution $\delta_0$, the Wasserstein-$p$ distance between any policy $\pi$ and $\tilde{\Pi}$ is given by:
$$d(\pi, \tilde{\Pi}) = W_p(P_{\pi}, \delta_0).$$
When the target distribution is a Dirac measure, the Wasserstein distance simplifies to:
$$W_p(P_{\pi}, \delta_0)^p = \int_{0}^{+\infty} \rho^p \, dP_{\pi}(\rho) = \mathbb{E}_{\tau \sim \pi} \left[ \rho(\tau)^p \right].$$
Thus, we need to prove:
$$\mathbb{E}_{\tau \sim \pi_1} \left[ \rho(\tau)^p \right] \geq \mathbb{E}_{\tau \sim \pi_2} \left[ \rho(\tau)^p \right].$$
From our assumptions, $P_{\pi_2}(\rho) = r(\rho) P_{\pi_1}(\rho)$, where $r(\rho)$ is a non-increasing function. This implies that for any $\rho_1 \leq \rho_2$:
$$\frac{P_{\pi_2}(\rho_1)}{P_{\pi_1}(\rho_1)} = r(\rho_1) \geq r(\rho_2) = \frac{P_{\pi_2}(\rho_2)}{P_{\pi_1}(\rho_2)}.$$
In other words, $\pi_2$ redistributes probability mass from smaller to larger $\rho$ values less aggressively than $\pi_1$.
Rearrangement Inequality states that if $f(\rho)$ is non-decreasing and $g(\rho)$ is non-increasing, then:
$$\int_{0}^{+\infty} f(\rho) g(\rho) d\rho \leq \int_{0}^{+\infty} f(\rho) \tilde{g}(\rho) d\rho,$$
where $\tilde{g}(\rho)$ is the non-increasing rearrangement of $g(\rho)$.
Application:

Let $f(\rho) = \rho^p$ (clearly non-decreasing).
Let $g(\rho) = P_{\pi_1}(\rho)$, and since $P_{\pi_2}(\rho) = r(\rho) P_{\pi_1}(\rho)$ with $r(\rho)$ non-increasing, $P_{\pi_2}(\rho)$ represents a non-increasing rearrangement of $P_{\pi_1}(\rho)$.

By the rearrangement inequality, we have:
\begin{align}
    \int_{0}^{+\infty} \rho^p P_{\pi_2}(\rho) d\rho 
    &= \int_{0}^{+\infty} \rho^p r(\rho) P_{\pi_1}(\rho) d\rho \notag \\
    &\leq \int_{0}^{+\infty} \rho^p P_{\pi_1}(\rho) d\rho. \notag
\end{align}
This establishes:
$$\mathbb{E}_{\tau \sim \pi_2} \left[ \rho(\tau)^p \right] \leq \mathbb{E}_{\tau \sim \pi_1} \left[ \rho(\tau)^p \right].$$
From the expression of Wasserstein-$p$ distance, we directly obtain:
$$W_p(P_{\pi_1}, \delta_0)^p \geq W_p(P_{\pi_2}, \delta_0)^p \implies d(\pi_1, \tilde{\Pi}) \geq d(\pi_2, \tilde{\Pi}).$$
\end{proof}

% The proof of \cref{lemma:distance_geq} is detailed in Appendix \ref{app:proof_lemma:distance_geq}. 

Based on \cref{lemma:ratio_monotonically} and \cref{lemma:distance_geq}, we can readily derive that:

If exist a trajectory return function $R(\tau)$ that is monotonically non-increasing with respect to the distance between a trajectory $\tau$ and the desired trajectory set $\tilde{\mathcal{T}}$, such that:
$$
d(\tau_1, \tilde{\mathcal{T}}) \leq d(\tau_2, \tilde{\mathcal{T}}) \implies R(\tau_1) \geq R(\tau_2).
$$
Suppose policy $\pi_2$ is obtained by optimizing policy $\pi_1$ using a maximum entropy algorithm with respect to $R$, 
Then, policy $\pi_2$ is closer to the optimal policy set $\tilde{\Pi}$ than $\pi_1$:
$$
d(\pi_1, \tilde{\Pi}) \geq d(\pi_2, \tilde{\Pi}).
$$

\section{Indicative Test Combination}
\label{app:indicative}
In this section, we explore the combination of multiple indicative tests into a return function. 

\cref{fig:pf_combine} shows the combination of pass-fail tests, where the optimal trajectory set $\tilde {\mathcal T}$ is the intersection of $\tilde {\mathcal T}_i, \forall i\in\{1,2,...m\}$. The policy optimization objective seeks to find a policy whose interaction with the environment generates trajectories that all fall within the optimal trajectory set $\tilde{\mathcal{T}}$. Therefore, we can construct a trajectory reward function such that trajectories closer to the optimal trajectory set $\tilde{\mathcal{T}}$ receive higher rewards.

\begin{figure}[htb]
\centering
\includegraphics[width=\columnwidth]{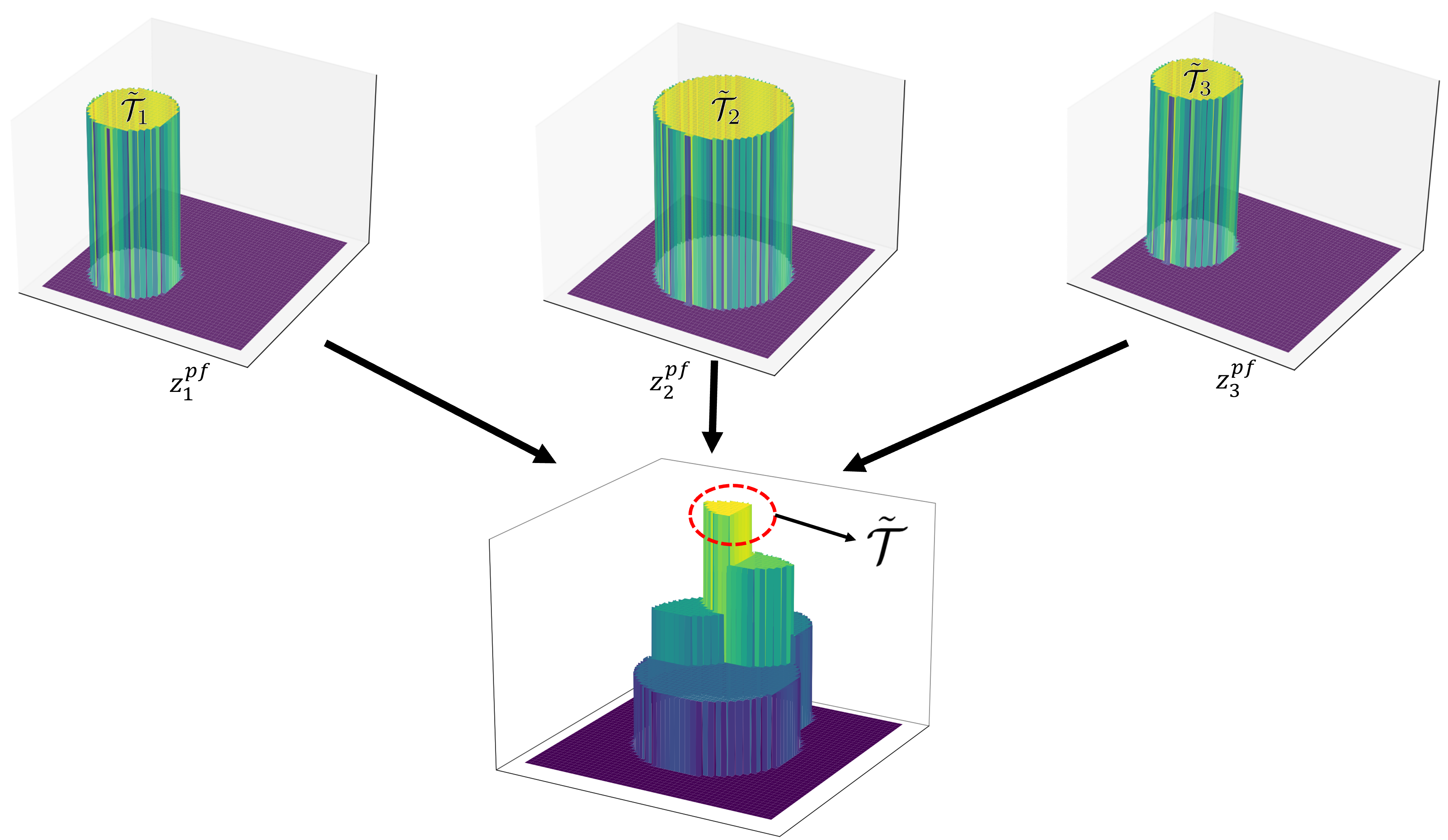}
\caption{A combination of pass-fail tests. The top row displays individual pass-fail test functions, with yellow regions denoting trajectories satisfying $z^{pf}_i(\tau)=1$ (passing test). The bottom row presents the summation function of these pass-fail tests, where the output value equals the count of passed tests. The maximal-value region (where $\sum_{i=1}^m z^{pf}_i(\tau)=m$) identifies trajectories belonging to $\tilde{\mathcal T}$ that satisfy all test conditions.}
\label{fig:pf_combine}
\end{figure}

The lower plot in \cref{fig:pf_combine} exhibits this property, but since its derivative vanishes almost everywhere, it cannot provide meaningful optimization directions. Thus, we design a trajectory reward function that offers such directional guidance. \cref{fig:ind_combine} illustrates this function: it not only assigns higher rewards to trajectories nearer to $\tilde{\mathcal{T}}$ but also provides gradient information usable for policy optimization.

\begin{figure}[htb]
\centering
\includegraphics[width=\columnwidth]{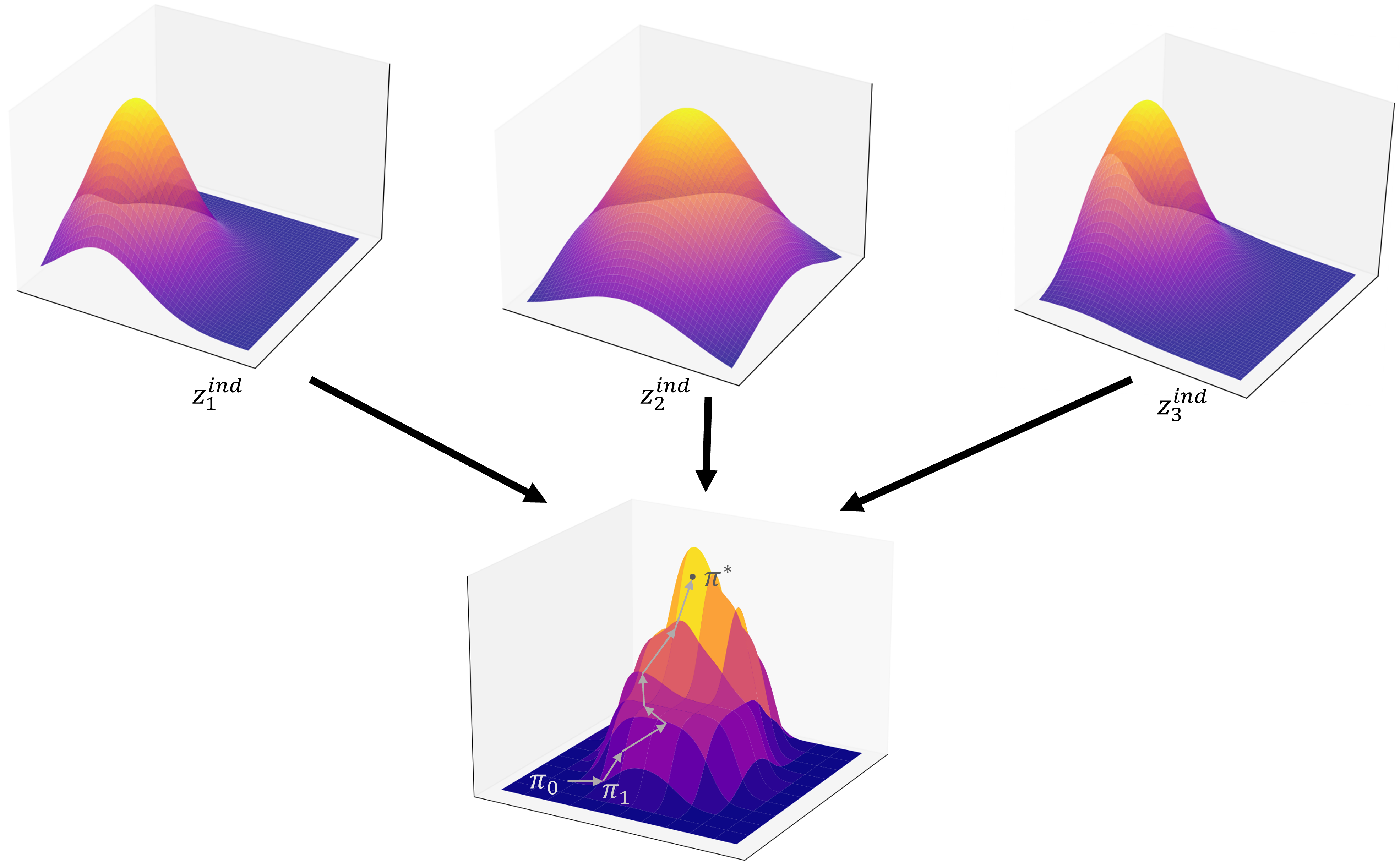}
\caption{A combination of indicative tests. The top row displays individual indicative test functions, where the vertical axis represents the test results of trajectories. The bottom row illustrates the trajectory return function obtained by combining these test functions. Policy optimization performed on this return function enables the agent to generate trajectories that converge to the optimal trajectory set (highlighted in \cref{fig:pf_combine}) through environment interactions.
}
\label{fig:ind_combine}
\end{figure}

A trajectory return function with the desired properties can be constructed by combining indicative test functions. \cref{fig:ind_combine} illustrates this composition process. However, constructing such a return function via linear weighted summation is challenging: trajectories satisfying all test criteria rarely excel in any individual metric, and linear combinations often induce metric conflicts or excessive emphasis on one metric. To address this, we use a nonlinear neural network to integrate the indicative test functions into the trajectory reward function.

\section{Interpretation of Trajectory Comparison}
\label{app:interpertation}

This section provides a detailed explanation of the following trajectory comparison priors:
\begin{itemize}
    \item Pass-fail tests take precedence over indicative tests
    \item A trajectory satisfying more pass-fail tests is closer to $\tilde{\mathcal{T}}$
    \item A trajectory passing more challenging tests (corresponding to smaller $\tilde{\mathcal{T}_i}$) is closer to $\tilde{\mathcal{T}}$
    \item All trajectories within $\tilde{\mathcal{T}}$ have zero distance to $\tilde{\mathcal{T}}$
    \item Under-optimized indicators should be prioritized
\end{itemize}

The optimization objective in TdRL requires trajectories to satisfy all pass-fail tests, while indicative tests provide quantitative measurements of specific aspects. In other words, pass-fail tests directly reflect the policy optimization goals, while indicative tests represent potential optimization pathways. This fundamental distinction establishes the priority hierarchy: \textbf{pass-fail tests take precedence over indicative tests during trajectory evaluation}.

In pass-fail tests, for the collection of sets $\{\tilde{\mathcal{T}_1}, \tilde{\mathcal{T}_2}, ..., \tilde{\mathcal{T}_m}\}$, consider their intersection family (i.e., all sets of the form $\bigcap_{i \in I} \tilde{\mathcal{T}_i}$ where $I \subseteq \{1,2,...,m\}$). There exists a minimal set $\tilde{\mathcal{T}_\tau}$ in this family that contains trajectory $\tau$, with $\tilde{\mathcal{T}} \subseteq \tilde{\mathcal{T}_\tau}$. The smaller $\tilde{\mathcal{T}_\tau}$, the tighter the upper bound on $\tilde{d}(\tau)$. The equality $\tilde{d}(\tau)=0$ holds iff $\tilde{\mathcal{T}_\tau}=\tilde{\mathcal{T}}$. In this intersection family, the intersection typically becomes smaller as the number of intersecting sets increases. Therefore, we could derive that \textbf{a trajectory satisfying more pass-fail tests is closer to $\tilde{\mathcal{T}}$}.

Similarly, if a pass-fail test is more difficult, the set of trajectories that can pass this test becomes smaller, thus we have that \textbf{a trajectory passing more challenging tests (corresponding to smaller $\tilde{\mathcal{T}_i}$) is closer to $\tilde{\mathcal{T}}$}. 
However, directly comparing the size of a difficult test's passing trajectory set with the intersection of multiple simple tests' passing sets is challenging. Thus, we cannot definitively determine their relative priorities. Given that our ultimate objective is to pass all pass-fail tests, this work prioritizes the quantity of passed tests over their difficulty. This prioritization can, of course, be adapted based on specific task requirements.

When pass-fail test functions cannot distinguish between trajectories (either all tests fail or trajectories pass identical test sets), we lose relative performance information between trajectories. In such cases, metric-based tests provide finer-grained differentiation. However, since multiple metric tests typically exist, their prioritization becomes necessary. Empirically, trajectories passing all pass-fail tests tend to exhibit consistently high metric test scores. Therefore, our training process emphasizes that \textbf{under-optimized indicators should be prioritized}, quantified using the skewness of historical metric test data distributions.

\cref{alg:lexicographic} presents the detailed pseudocode for the lexicographic trajectory comparison method.

\begin{algorithm}[htbp]
\caption{TdRL}
\label{alg:lexicographic}
\textbf{Input}: $\tau_1$, $\tau_2$
\textbf{Output}: $\mu$
\begin{algorithmic}[1] %[1] enables line numbers
\IF{$\sum_{i=1}^m z^{pf}_i(\tau_1) = \sum_{i=1}^m z^{pf}_i(\tau_2) = m$}
\STATE return $\mu=0.5$
\ENDIF
\IF{$\sum_{i=1}^m z^{pf}_i(\tau_1) > \sum_{i=1}^m z^{pf}_i(\tau_2)$}
\STATE return $\mu=1$ 
\ENDIF
\IF {$\sum_{i=1}^m z^{pf}_i(\tau_1) < \sum_{i=1}^m z^{pf}_i(\tau_2)$}
\STATE return $\mu=0$
\ENDIF
\FOR{$i=1$ to $m$}
\IF{$z^{pf}_{k_i}(\tau_1) > z^{pf}_{k_i}(\tau_2)$}
\STATE return $\mu=1$
\ENDIF
\ENDFOR
\FOR{$i=1$ to $n$}
\IF{$z^{ind}_{l_i}(\tau_1) > z^{ind}_{l_i}(\tau_2)$}
\STATE return $\mu=1$
\ENDIF
\ENDFOR
\IF{$\mu$ is not assigned yet}
\STATE return $\mu=0.5$
\ENDIF
\end{algorithmic}
\end{algorithm}

\section{Test Function Settings}
\label{app:tester}

This section provides a detailed description of the test function settings for each task in the experimental part of the paper. To standardize notation, we adopt the prefix "pf-" for pass-fail tests and "ind-" for indicative tests.

\begin{itemize}
    % \item \textbf{Cartpole-Balance}
    % \begin{itemize}
    %     \item \textbf{pf-upright}: verifies if the cosine of pole angle remains within $[0.995, 1]$ at all timesteps.
    %     \item \textbf{pf-pos}: verifies if the cart's horizontal position stays within $[-0.25, 0.25]$ at all timesteps.
    %     \item \textbf{ind-upright}: counts the number of timesteps where the pole's cosine angle falls within $[0.995, 1]$
    %     \item \textbf{ind-pos}: counts the number of timesteps where the cart's horizontal position is within $[-0.25, 0.25]$
    % \end{itemize}
    \item \textbf{Walker-Stand}
    \begin{itemize}
        \item \textbf{pf-upright}: verifies if the cosine of torso angle remains within $[0.9, 1]$ at all timesteps.
        \item \textbf{pf-height}: verifies if the torso's height stays within $[1.2, +\infty]$ at all timesteps.
        \item \textbf{ind-upright}: counts the number of timesteps where the cosine of torso angle remains within $[0.9, 1]$.
        \item \textbf{ind-height}: counts the number of timesteps where the torso's height stays within $[1.2, +\infty]$.
    \end{itemize}
    \item \textbf{Walker-Run}
    \begin{itemize}
        \item \textbf{pf-upright}: verifies if the cosine of torso angle remains within $[0.9, 1]$ at all timesteps.
        \item \textbf{pf-height}: verifies if the torso's height stays within $[1.2, +\infty]$ at all timesteps.
        \item \textbf{pf-speed}: verifies if the mean speed in x-axis of torso stays within $[8, +\infty]$ across the trajectory.
        \item \textbf{ind-upright}: counts the number of timesteps where the cosine of torso angle remains within $[0.9, 1]$.
        \item \textbf{ind-height}: counts the number of timesteps where the torso's height stays within $[1.2, +\infty]$.
        \item \textbf{ind-speed}: mean speed in x-axis across the trajectory.
    \end{itemize}

    \item \textbf{Walker-JumpRun}
    \begin{itemize}
        \item \textbf{pf-upright}: verifies if the cosine of torso angle remains within $[0.9, 1]$ at all timesteps.
        \item \textbf{pf-height}: verifies if the torso's height stays within $[1.2, +\infty]$ at all timesteps.
        \item \textbf{pf-jump}: verifies if the torso's max height stays within $[1.5, +\infty]$ in the trajectory.
        \item \textbf{pf-speed}: verifies if the mean speed in x-axis of torso stays within $[8, +\infty]$ across the trajectory.
        \item \textbf{ind-upright}: counts the number of timesteps where the cosine of torso angle remains within $[0.9, 1]$.
        \item \textbf{ind-height}: counts the number of timesteps where the torso's height stays within $[1.2, +\infty]$.
        \item \textbf{ind-jump}: the torso's max height in the trajectory.
        \item \textbf{ind-speed}: mean speed in x-axis across the trajectory.
    \end{itemize}

    \item \textbf{Cheetah-Run}
    \begin{itemize}
        \item \textbf{pf-speed}: verifies if the mean speed in x-axis of body stays within $[5, +\infty]$ across the trajectory.
        \item \textbf{ind-speed}: mean speed in x-axis across the trajectory.
    \end{itemize}
    \item \textbf{Quadruped-Run}
    \begin{itemize}
        \item \textbf{pf-upright}: verifies if the cosine of torso angle remains within $[0.9, 1]$ at all timesteps.
        \item \textbf{pf-speed}: verifies if the mean speed in x-axis of torso stays within $[10, +\infty]$ across the trajectory.
        \item \textbf{ind-upright}: counts the number of timesteps where the cosine of torso angle remains within $[0.9, 1]$.
        \item \textbf{ind-speed}: mean speed in x-axis across the trajectory.
    \end{itemize}
    
\end{itemize}

\end{document}